\documentclass{article}

\usepackage{microtype}
\usepackage{graphicx}
\usepackage{booktabs} 

\usepackage[accepted]{icml2024}

\usepackage{amsmath}
\usepackage{amssymb}
\usepackage{mathtools}
\usepackage{amsthm}

\theoremstyle{plain}

\newtheorem{definition}{Definition}[section]
\newtheorem{remark}{Remark}[section]

\newcommand{\Id}{\text{\it Id}}

\usepackage{hyperref}
\usepackage{url}
\usepackage{adjustbox}
\usepackage[raggedright]{sidecap}

\usepackage{array}
\usepackage{makecell}
\usepackage{thmtools}

\usepackage{amssymb, amsmath, amsthm}

\definecolor{shadecolor}{gray}{0.9}

\usepackage{enumitem}

\usepackage{amsthm}

\sidecaptionvpos{figure}{t} 
\usepackage{graphicx}
\usepackage[labelfont=bf]{caption}
\usepackage[format=hang]{subcaption}

\usepackage{microtype}

\usepackage{booktabs}

\usepackage[algoruled, algo2e]{algorithm2e}
\usepackage{algorithm}
\usepackage{algorithmic}

\usepackage{listings}
\usepackage{fancyvrb}
\usepackage{natbib}

\usepackage{xcolor}
\hypersetup{
    colorlinks,
    linkcolor={red!50!black},
    citecolor={blue!50!black},
    urlcolor={blue!80!black}
}

\usepackage[nameinlink, capitalise]{cleveref}
\usepackage[acronym,nowarn]{glossaries}
\glsdisablehyper

\usepackage{nccmath}

\usepackage[theorems,skins]{tcolorbox}

\usepackage{wrapfig}

\usepackage{nicefrac}

\usepackage{multirow}
\usepackage{soul}
\usepackage{color}

\definecolor{amaranth}{rgb}{0.9, 0.17, 0.31}

\definecolor{airforceblue}{rgb}{0.36, 0.54, 0.66}

\newacronym{mcmc}{MCMC}{Markov chain Monte Carlo}
\newacronym{ode}{ODE}{ordinary differential equation}
\newacronym{sde}{SDE}{stochastic differential equation}
\newacronym{gmm}{PDF}{Gaussian mixture model}
\newacronym{pdf}{PDF}{probability density function}

\usepackage{amsmath,amsfonts,bm}

\def\1{\bm{1}}

\def\eps{{\epsilon}}

\DeclareMathAlphabet{\mathsfit}{\encodingdefault}{\sfdefault}{m}{sl}
\SetMathAlphabet{\mathsfit}{bold}{\encodingdefault}{\sfdefault}{bx}{n}

\newcommand{\E}{\mathbb{E}}

\newcommand{\R}{\mathbb{R}}

\icmltitlerunning{Stochastic Interpolants with Data-Dependent Couplings}

\begin{document}

\twocolumn[
\icmltitle{Stochastic Interpolants with Data-Dependent Couplings}



\icmlsetsymbol{equal}{*}

\begin{icmlauthorlist}
\icmlauthor{Michael S.~Albergo}{equal,1}
\icmlauthor{Mark Goldstein}{equal,2}
\icmlauthor{Nicholas M. Boffi}{2}
\icmlauthor{Rajesh Ranganath}{2,3}
\icmlauthor{Eric Vanden-Eijnden}{2}
\end{icmlauthorlist}
\icmlaffiliation{1}{Center for Cosmology and Particle Physics, New York University}
\icmlaffiliation{2}{Courant Institute of Mathematical Sciences, New York University}
\icmlaffiliation{3}{Center for Data Science, New York University}
\icmlcorrespondingauthor{Michael S.~Albergo}{albergo@nyu.edu}
\icmlcorrespondingauthor{Mark Goldstein}{goldstein@nyu.edu}

\icmlkeywords{Machine Learning, ICML}

\vskip 0.3in
]



\printAffiliationsAndNotice{\icmlEqualContribution} 

\begin{abstract}
Generative models inspired by dynamical transport of measure -- such as flows and diffusions -- construct a continuous-time map between two probability densities. 
Conventionally, one of these is the target density, only accessible through samples, while the other is taken as a simple base density that is data-agnostic.
In this work, using the framework of stochastic interpolants, we formalize how to \textit{couple} the base and the target densities, whereby samples from the base are computed conditionally given samples from the target in a way that is different from (but does not preclude) incorporating information about class labels or continuous embeddings.
This enables us to construct dynamical transport maps  that serve as conditional generative models.
We show that these transport maps can be learned by solving a simple square loss regression problem analogous to the standard independent setting.
We demonstrate the usefulness of constructing dependent couplings in practice through experiments in super-resolution and in-painting.
The code is available at 
\href{https://github.com/interpolants/couplings}{https://github.com/interpolants/couplings}.
\end{abstract}

\section{Introduction}
\label{sec:intro}

\begin{figure}[ht]
  \begin{center}
\includegraphics[width=0.40\textwidth]{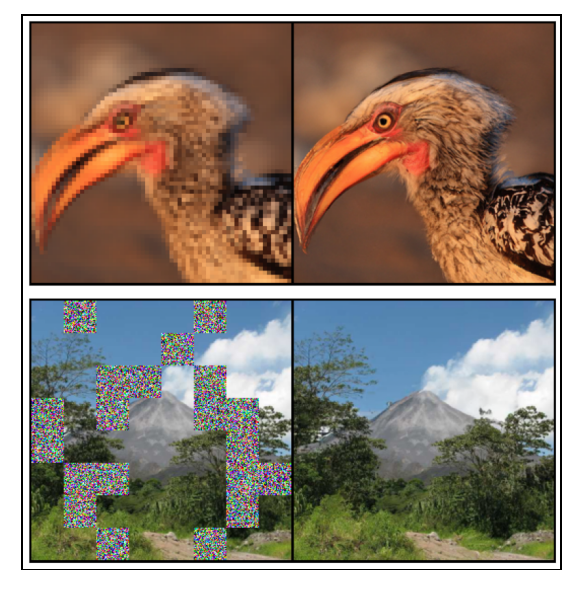}
\vspace{-0.4cm}
    \caption{\textbf{Examples.} Super-resolution and in-painting results computed with our formalism.}
  \end{center}
     \label{fig:front_page}
\end{figure}

Generative models such as normalizing flows and diffusions sample from a target density $\rho_1$ by continuously transforming samples from a base density $\rho_0$ into the target.
This transport is accomplished by means of an 
\gls{ode} or \gls{sde}, which takes as initial condition a sample from $\rho_0$ and produces at time $t=1$ an approximate sample from $\rho_1$.
Typically, the base density is taken to be something simple, analytically tractable, and easy to sample, such as a standard Gaussian.
In some formulations, such as score-based diffusion
\citep{sohl2015deep,
song2020improved,
ho2020denoising,
song2020score, singhal2023diffuse}, a Gaussian base density is intrinsically tied to the process achieving the transport.  In others, including flow matching
\citep{lipman2022flow,chen2023riemannian},  rectified flow
\citep{liu2022flow, liu2023instaflow}, and stochastic interpolants
\citep{albergo2022building, albergo2023stochastic},  a Gaussian base is not required, but is often chosen for convenience.
In these cases, the choice of Gaussian base represents an absence of prior knowledge about the problem structure, and existing works have yet to fully explore the strength of base densities adapted to the target.

In this work, we introduce a general formulation  of stochastic interpolants in which a base density is produced via a \textit{coupling}, whereby samples of this base are computed conditionally given samples from the target.
We construct a continuous-time stochastic process that interpolates between the coupled base and target, and we characterize the resulting transport by identification of a continuity equation obeyed by the time-dependent density.
We show that the velocity field defining this transport can be estimated by solution of an efficient, simulation-free square loss regression problem analogous to standard, data-agnostic interpolant and flow matching algorithms.

In our formulation, we also allow for dependence on an external, conditional source of information independent of $\rho_1$, which we call $\xi$. This extra source of conditioning is standard, and can be used in the velocity field $b_t(x,\xi)$ to accomplish class-conditional generation, or generation conditioned on a continuous embedding such as a textual representation or problem-specific geometric information. As illustrated in Fig.~\ref{fig:pictorial}, it is however different from the data-dependent coupling that we propose. 
Below, we suggest some generic ways to construct coupled, conditional base and target densities, and we consider practical applications to image super-resolution and in-painting, where we find improved performance by incorporating both a data-dependent coupling and the conditioning variable.
\begin{figure*}[ht]
    \centering
    \includegraphics[width=0.88\linewidth]{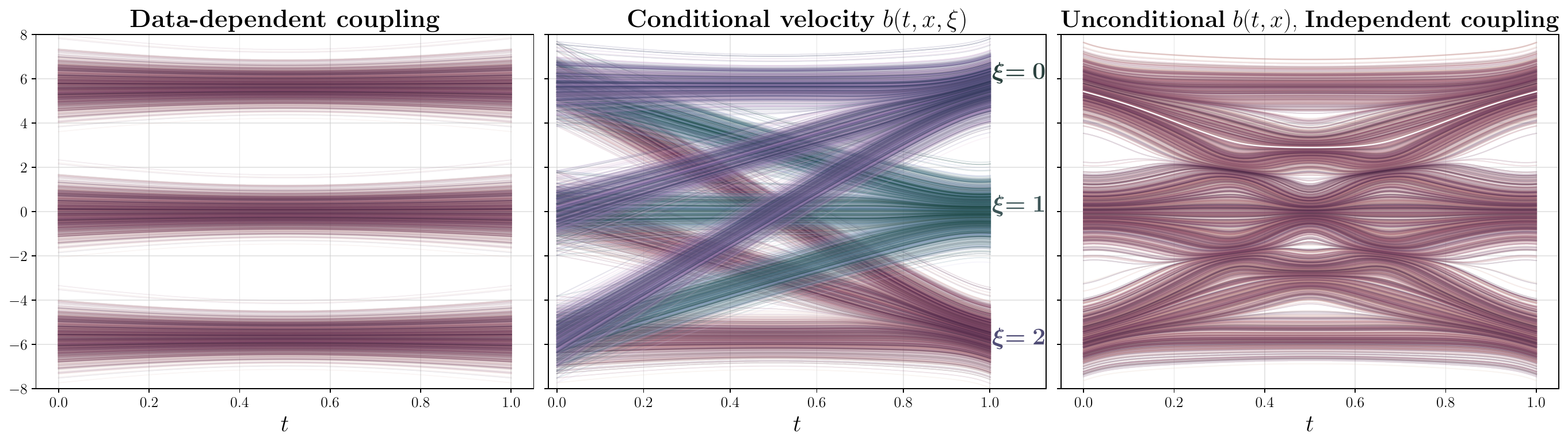}
    \caption{\textbf{Data-dependent couplings are different than conditioning.}
    Delineating between constructing couplings versus conditioning the velocity field, and their implications for the corresponding probability flow $X_t$. The transport problem is flowing from a Gaussian Mixture Model (GMM) with 3 modes to another GMM with 3 modes.
    \textit{Left}: The probability flow $X_t$ arising from the data-dependent coupling $\rho(x_0, x_1) = \rho_1(x_1)\rho_0(x_0 | x_1)$. All samples follow simple trajectories. No formation of auxiliary modes form in the intermediate density $\rho(t)$, in juxtaposition to the independent case. 
    \textit{Center}: When the velocity field is conditioned $b_t(x,\xi)$ on each class (mode), it factorizes, resulting in three separate probability flows $X_t^{\xi}$ with $\xi = 1,2,3$.
    \textit{Right}: The probability flow $X_t$ when taking an unconditional velocity field $b_t(x)$ and an independent coupling $\rho(x_0, x_1) = \rho_0(x_0)\rho_1(x_1)$. Note the complexity of the underlying transport, which motivates us to consider finding correlated base variables directly in the data.
    }
    \label{fig:pictorial}
\end{figure*}
Together, our \textbf{main contributions} can be summarized as:
\begin{enumerate}
    \item We define a broader way of constructing base and target pairs in generative models based on dynamical transport that adapts the base to the target. 
    In addition, we formalize the use of conditional information -- both discrete and continuous -- in concert with this new form of \textit{data coupling} in the stochastic interpolant framework.
    As special cases of our general formulation, we obtain several recent variants of conditional generative models that have appeared in the literature.
    \item We provide a  characterization of the transport that results from conditional, data-dependent generation, and analyze theoretically how these factors influence the resulting time-dependent density
    \item We provide an empirical study on the effect of coupling for stochastic interpolants, which have recently been shown to be a promising, flexible class of generative models.  We demonstrate the utility of data-dependent base densities and the use of conditional information in two canonical applications, image inpainting and super-resolution, which highlight the performance gains that can be obtained through the application of the tools developed here   . 
\end{enumerate}
The rest of the paper is organized as follows.
In~\cref{sec:related}, we describe some related work in conditional generative modeling.
In~\cref{sec:theo}, we introduce our theoretical framework. We characterize the transport that results from the use of data-dependent couplings, and discuss the difference between this approach and conditional generative modeling.
In~\cref{sec:experiments}, we apply the framework to numerical experiments on ImageNet, focusing on image inpainting and image super-resolution.
We conclude with some remarks and discussion in~\cref{sec:disc}.

\section{Related Work \label{sec:related}}
\begin{table*}[bt]
\caption{
\textbf{Couplings.}
Standard formulations of flows and diffusions construct generative models built upon an independent coupling \citep{albergo2022building,albergo2023stochastic, lipman2022flow, liu2022flow}. 
\cite{lee2023minimizing} learn $q_\phi(x_0|x_1)$ jointly with the velocity to define the coupling during training, but instead sample from $\rho_0={\sf N}(0,\Id)$ for generation.
\cite{tong2023improving} and \cite{pooladian2023multisample} build couplings by running mini-batch optimal transport algorithms 
\citep{cuturi2013sinkhorn}. 
Here we focus on couplings enabled by our generic formalism, which bears similarities with \cite{liu20232,somnath2023aligned}, and can be individualized to each generative task.}
\label{tab:couplings}
\begin{center}
\begin{tabular}{lll}
\toprule
\multicolumn{1}{c}{\bf Coupling PDF~$\rho(x_0,x_1)$}
&\multicolumn{1}{c}{\bf Base PDF}
&\multicolumn{1}{c}{\bf Description}
\\ \midrule 
$\rho_1(x_1)\rho_0(x_0)$ & $x_0 \sim {\sf N}(0,\Id)$ & Independent\\
$\rho(x_0|x_1)\rho_1(x_1)$ & $x_0 \sim q_\phi(x_0|x_1)$
&  Learned conditional\\
 $\text{mb-OT}(x_1, x_0)$ &  $x_0\sim {\sf N}(0,\Id)$
 & Minibatch OT
\\ 
\midrule 
$ \rho_1(x_1)\rho_0(x_0|x_1)$   
&  $x_0 \sim \rho_0(x_0|x_1)$ & Dependent-coupling \textbf{(this work)}\\
     \bottomrule
\end{tabular}
\end{center}
\end{table*}

\paragraph{Couplings.} 

Several works have studied the question of how to build couplings, primarily from the viewpoint of optimal transport theory. 
An initial perspective in this regard comes from \cite{pooladian2023multisample, tong2023improving, klein2023equivariant}, who state an unbiased means for building entropically-regularized optimal couplings from minibatches of training samples. 
This perspective is appealing in that it may give probability flows that are straighter and hence more easily computed using simple \gls{ode} solvers. However, it relies on estimating an optimal coupling over minibatches of the entire dataset, which, for large datasets, may become uninformative as to the true coupling. In an orthogonal perspective, \cite{lee2023minimizing} presented an algorithm to learn a coupling between the base and the target by building dependence on the target into the base. They argue that this can reduce curvature of the underlying transport. While this perspective empirically reduces the curvature of the flow lines, it introduces a potential bias in that they still sample from an independent base, possibly not equal to the marginal of the learned conditional base. Learning a coupling can also be achieved by solving the Schr\"odinger bridge problem, as investigated e.g. in \cite{bortoli2021,shi2023diffusion}. This leads to iterative algorithms that require solving pairs of SDEs until convergence, which is costly in practice. 
More closely connected to our work are the approaches proposed in \cite{liu20232,somnath2023aligned}: by considering generative modeling through the lens of diffusion bridges with \textit{known} coupling, they arrive to a formulation that is operationally similar to, but less general than,  ours. Our approach is simpler, and more flexible, as it differentiates between the bridging of the densities and the construction of the generative models. \Cref{tab:couplings} summarizes these couplings along with the standard independent pairing.

\paragraph{Generative Modeling and Dynamical Transport.}

Generative models built upon dynamical transport of measure go back at least to~\citep{tabak_density_2010, tabak2013}, and were further developed in~\citep{rezende_variational_2015, dinh_density_2017, huang_deep_2016, durkan2019} using compositions of discrete maps, while modern models are typically formulated via a continuous-time transformation. 
In this context, a major advance was the introduction of score-based diffusion~\citep{song_scorebased_2021, song2021mle}, which relates to denoising diffusion probabilistic models~\citep{ho2020}, and allows one to generate samples by learning to reverse a stochastic differential equation that maps the data into samples from a Gaussian base density.
Methods such as flow matching~\citep{lipman2022}, rectified flow~\citep{liu2022-ot, liu2022}, and stochastic interpolants~\citep{albergo2022building, albergo2023stochastic} expand on the idea of  building stochastic processes that connect a base density to the target, but allow for bases that are more general than a Gaussian density. 
Typically, these constructions assume that the samples from the base and the target are uncorrelated.

\paragraph{Conditional Diffusions and Flows for Images.}

\cite{saharia2022image, ho2022cascaded} build diffusions for super-resolution, where low-resolution images are given as inputs to a score model,
which formally learns a conditional score \citep{ho2022classifier}.
In-painting can be seen as a form of conditioning where the conditioning set determines some coordinates in the target space.
In-painting diffusions have been applied to video generation \citep{ho2022video} and protein backbone generation \citep{trippe2022diffusion}. In the \textit{replacement method} 
 one directly inputs the clean values of the known coordinates at each step of integration \citep{ho2022video}; \cite{schneuing2022structure} replace with draws of the diffused state of the known coordinates.
\cite{trippe2022diffusion, wu2023practical} discuss approximation error
in this approach and correct with sequential Monte-Carlo. We revisit this problem framing from the velocity modeling perspective 
in  \cref{sec:inp}.
Recent work has applied
flows to high-dimensional conditional modeling
\citep{dao2023flow,hu2023latent}. A Schr\"odinger bridge perspective on the conditional generation problem was presented in \citep{shi2022conditional}.

\section{Stochastic interpolants with couplings}
\label{sec:theo}

Suppose that we are given a dataset $\{x_1^i\}_{i=1}^n$.
The aim of a generative model is to draw new samples assuming that the data set comes from a 
\gls{pdf} $\rho_1(x_1) $.
Following the stochastic interpolant framework~\citep{albergo2022building, albergo2023stochastic}, we introduce a time-dependent stochastic process that interpolates between samples from a simple base density $\rho_0(x_0)$  at time $t=0$ and samples from  the target $\rho_1(x_1)$ at time $t=1$:
\begin{definition}[Stochastic interpolant with coupling]
\label{def:cond:interp}
The  stochastic interpolant $I_t$ is the process defined as\footnote{%
More generally, we may set $I_t= I(t,x_0,x_1)$ in~\eqref{eq:stochinterpolant}, where $I$ satisfies some regularity properties in addition to the boundary conditions $I(t=0,x_0,x_1)=x_0$ and $I(t=1,x_0,x_1)=x_1$~\citep{albergo2022building,albergo2023stochastic}. 
For simplicity, we will stick to the linear choice $I(t,x_0,x_1) = \alpha_t x_0 + \beta_t  x_1$.} 
\begin{equation}
    \label{eq:stochinterpolant}
    I_t = \alpha_t  x_0 + \beta_t  x_1 + \gamma_t  z\qquad t\in[0,1],
\end{equation}
where
\begin{itemize}[leftmargin=0.15in]
\item  $\alpha_t $, $\beta_t $, and $\gamma^2_t$ are differentiable functions of time such that $\alpha_0=\beta_1 =1$,  $\alpha_1=\beta_0 =\gamma_0 = \gamma_1=0$, and $\alpha^2_t+\beta^2_t+\gamma^2_t>0$ for all $t\in[0,1]$.
\item The pair $(x_0,x_1)$ is jointly drawn from a probability density $ \rho(x_0,x_1)$ with finite second moments and such that 
   \begin{align}
       \label{eq:rhoc:cond}
       \int_{\R^d} \rho(x_0,x_1) dx_1 = \rho_0(x_0), \\ \int_{\R^d} \rho(x_0,x_1) dx_0 =  \rho_1(x_1).
   \end{align}
   \item $z \sim \mathsf N(0,\Id)$, independent of $(x_0,x_1)$.
\end{itemize}
\end{definition}

A simple instance of \eqref{eq:stochinterpolant} uses $\alpha_t  = 1-t$, $\beta_t =t$, and $\gamma_t  = \sqrt{2t(1-t)}$.

The stochastic interpolant framework uses information about the process $I_t$ to derive either an \gls{ode} or an \gls{sde}
 whose solutions $X_t$ push the law of $x_0$ onto the law of $I_t$ for all times $t\in[0, 1]$.

As shown in~\cref{sec:con:tranp}, the drift coefficients in these \glspl{ode}/\glspl{sde} can be estimated by quadratic regression.
They can then be used as generative models, owing to the property that the process $x_t$ specified in~\cref{def:cond:interp} satisfies $I_{t=0}=x_0 \sim \rho_0(x_0)$ and $I_{t=1}=x_1 \sim \rho_1(x_1)$, and hence samples the desired target density.
By drawing samples $x_0\sim \rho_0(x_0)$ and using them as initial data $X_{t=0} = x_0$ in the \glspl{ode}/\glspl{sde}, we can then generate samples $X_{t=1}\sim \rho_1(x_1)$ via numerical integration.

In the original stochastic interpolant papers, this construction was made using the choice $\rho(x_0,x_1) = \rho_0(x_0) \rho_1(x_1)$, so that $x_0$ and $x_1$ were drawn independently from the base and the target.

\emph{Our aim here is to build generative models that are more powerful and versatile by exploring and exploiting dependent couplings between $x_0$ and $x_1$ via suitable definition of $\rho(x_0,x_1)$. }
\begin{remark}[Incorporating conditioning]
    \label{rem:1}
    Our formalism allows (but does not require) that each data point $x_1^i\in \R^d$  comes with a label $\xi_i\in D$, such as a discrete class or a continuous embedding like that of a text caption. In this setup, our results can be straightforwardly generalized  by making all the quantities (PDF, velocities, etc.) conditional on $\xi$. This is discussed in Appendix~\ref{sec:proofs} and used in various forms in our numerical examples.
\end{remark}

\subsection{Transport equations and conditional generative models}
\label{sec:con:tranp}

In this section, we show that the probability distribution of the process $I_t$ defined in~\eqref{eq:stochinterpolant} has a time-dependent density  $\rho_t(x)$ that interpolates between $\rho_0(x)$ and $\rho_1(x)$.
We characterize this density as the solution of a transport equation, and we show that both the corresponding velocity field and the score $\nabla\log\rho_t(x)$ are minimizers of simple quadratic objective functions.

This result enables us to construct conditional generative models by approximating the velocity (and possibly the score) via minimization over a rich parametric class such as neural networks.
We first define the functions:
\begin{align}
    \label{eq:g:c}
    b_t(x) = \E(\dot I_t|I_t = x), \quad  g_t(x) = \E(z|I_t = x),
\end{align}
where the dot denotes time-derivative and $\E(\cdot|I_t=x)$ denotes the expectation over $\rho(x_0,x_1)$ conditional on $I_t=x$. We then have,
\begin{restatable}[Transport equation with coupling]{theorem}{cont0}
    \label{th:pdf0}
    The probability distribution of the stochastic interpolant $I_t$ defined in~\eqref{eq:stochinterpolant} has a density $\rho_t(x)$ that satisfies $\rho_{t=0}(x)= \rho_0(x)$ and $ \rho_{t=1}(x) = \rho_1(x)$, and solves the transport equation 
    \begin{align}
        \label{eq:tranp:fpe:c}
        \partial_t \rho_t(x)+ \nabla\cdot\left( b_t(x) \rho_t(x)\right) = 0,
    \end{align}
    where the velocity field $b_t(x)$ is defined in~\eqref{eq:g:c}.
    Moreover, for every $t$ such that $\gamma_t  \neq 0$, the following identity for the score holds
    \begin{equation}
        \label{eq:s:c}
        \nabla \log \rho_t(x)= -\gamma^{-1}_t g_t(x).
    \end{equation}
    Finally, the functions $b$ and $g$ are the unique minimizers of the objectives
        \begin{equation}
        \label{eq:obj}
        \begin{aligned}
            L_b(\hat b) &= \int_0^1 \E \left[|\hat b_t(I_t)|^2 - 2\dot I_t\cdot \hat b_t(I_t)\right] dt,\\
            L_g(\hat g) &= \int_0^1 \E \left[|\hat g_t(I_t)|^2 - 2z\cdot \hat g_t(I_t)\right] dt
        \end{aligned}
    \end{equation}
where   $\E$ denotes an expectation over $(x_0,x_1)\sim \rho(x_0,x_1)$ and $z\sim\mathsf N(0,\Id)$ with $(x_0,x_1)\perp z$.
\end{restatable}
A more general version of this result with a conditioning variable is proven in~\cref{sec:proofs}.
The objectives~\eqref{eq:obj} can readily be estimated in practice from samples $(x_0,x_1)\sim \rho(x_0,x_1)$ and $z \sim \mathsf N(0, 1)$, which will enable us to learn approximations for use in a generative model.

The transport equation~\eqref{eq:tranp:fpe:c} can be used to derive generative models, as we now show.
\begin{restatable}[Probability flow and diffusions with coupling]{corollary}{pflow0}
\label{th:gen:mod0} 
The solutions to the probability flow equation 
\begin{equation}
    \label{eq:prob:flow}
    \dot X_t = b_t(X_t)
\end{equation}
enjoy the property that
\begin{align}
X_{t=1} &\sim \rho_1(x_1)   
\quad \text{ if } \quad X_{t=0} \sim \rho_0(x_0)\\
X_{t=0} &\sim \rho_0(x_0)
\quad \text{ if } \quad X_{t=1} \sim \rho_1(x_1)
\end{align}
In addition, for any $\eps_t\ge0$, solutions to the forward \gls{sde}
 \begin{equation}
    \label{eq:sde:f}
    d X^F_t = b_t(X^F_t) dt - \eps_t \gamma^{-1}_tg_t(X^F_t) dt + \sqrt{2\eps_t} dW_t,
\end{equation}
enjoy the property that
\begin{align}
 X^F_{t=1} \sim \rho_1(x_1) \quad \text{if} \quad X^F_{t=0} \sim \rho_0(x_0),
\end{align}
and solutions to the backward \gls{sde}
 \begin{equation}
        \label{eq:sde:r}
        d X^R_t = b_t(X^R_t) dt + \eps_t \gamma^{-1}_tg_t(X^R_t) dt + \sqrt{2\eps_t} dW_t,
    \end{equation}
enjoy the property that
\begin{align}
 X^R_{t=0} \sim \rho_0(x_0) \quad 
 \text{if} \quad
 X^R_{t=1} \sim \rho_1(x_1).
\end{align}
\end{restatable}
A more general version of this result with conditioning is also proven in~\cref{sec:proofs}.

\Cref{th:gen:mod0} shows that the coupling can be incorporated  both in deterministic and stochastic generative models derived within the stochastic interpolant framework. In what follows, for simplicity we will focus on the deterministic probability flow ODE~\eqref{eq:prob:flow}.

An important observation is that the transport cost of the generative model based on the probability flow ODE~\eqref{eq:prob:flow}, which impacts the numerical stability of solving this ODE, is controlled by the time dynamics of the interpolant, as shown by our next result:
\begin{restatable}[Control of transport cost]{proposition}{cost}
    \label{th:transport}
    Let $X_t(x_0)$ be the solution to the probability flow ODE~\eqref{eq:prob:flow} for the initial condition~$X_{t=0} (x_0) = x_0 \sim \rho_0$. Then 
    \begin{equation}
    \label{eq:W2:bound}
        \E_{x_0\sim \rho_0}  \big[|X_{t=1}(x_0)-x_0|^2 \big] \le \int _0^1 \E[|\dot I_t|^2] dt <\infty
    \end{equation}
\end{restatable}
The proof of this proposition is given in ~\cref{sec:proofs}. Minimizing the left hand-side of~\eqref{eq:W2:bound} would achieve optimal transport in the sense of Benamou-Brenier \cite{benamou2000}, and the minimum would give the Wasserstein-2 distance between $\rho_0$ and $\rho_1$. Various works seek to minimize this distance procedurally either by adapting the coupling \cite{pooladian2023multisample, tong2023improving} or by optimizing $\rho_t(x)$ \cite{albergo2022building}, at additional cost. Here we introduce \textit{designed} couplings at no extra cost that can lower the upper bound in \eqref{eq:W2:bound}. This will allow us to show how different couplings enable stricter control of the transport cost in various applications. Let us now discuss a  generic instantiation of our formalism involving a specific choice of $\rho(x_0,x_1)$.

\subsection{Designing data-dependent couplings}
\label{sec:data}
One natural way to allow for a data-dependent coupling between the base and the target is to set
\begin{align}
    \label{eq:data:cond}
    \rho(x_0,x_1) = \rho_1(x_1 ) \rho_0(x_0|x_1 ) \quad \text{with} \\ \int_{\R^d}  \rho_0(x_0|x_1 )\rho_1(x_1 ) dx_1 = \rho_0(x_0 ).
\end{align}
There are many ways to construct the conditional $\rho_0(x_0|x_1)$.
%
In the numerical experiments in~\cref{sec:inp}~\&~\cref{sec:super}, we consider base densities of a variable $x_0$ of the generic form
\begin{equation}
    \label{eq:initial:x0}
    x_0 = m(x_1) + \sigma \zeta,
\end{equation}  
where $m(x_1)\in \R^d$ is some function of $x_1$, possibly random even if conditioned on $x_1$, $\sigma\in \R^{d \times d}$, and $\zeta\sim \mathsf N(0,\Id)$ with $\zeta \perp m(x_1)$.
In this set-up, the corrupted observation $m(x_1)$ (a noisy, partial, or low-resolution image) is determined by the task at hand and available to us, but we are free to choose the design of the term $\sigma\zeta$ in~\eqref{eq:initial:x0} in  ways that can be exploited differently in various applications (and is allowed to depend on any conditional info $\xi$).  Note in particular that, given $m(x_1)$, \eqref{eq:initial:x0} is easy to generate at sampling time. Note also that, if the corrupted observation $m(x_1)$ is deterministic given $x_1$, the conditional probability density of~\eqref{eq:initial:x0} is the Gaussian density with mean $m(x_1)$ and covariance $C=\sigma\sigma^\top$:
\begin{equation}
    \label{eq:Gauss}
    \rho_0(x_0|x_1) = \mathsf N(x_0;m(x_1),C),
\end{equation}
We stress that, even in this case,  $\rho(x_0,x_1) = \rho_1(x_1 ) \rho_0(x_0|x_1 )$ and $\rho_0(x_0) =  \rho_0(x_0|x_1)$ are non-Gaussian densities in general.
In this context, we can use the interpolant from 
\eqref{eq:stochinterpolant}
with $\gamma_t=0$, which reduces to:
\begin{align}
\label{eq:interp:exp}
    I_t &=
    \alpha_t (m(x_1) + 
    \sigma \zeta ) + \beta_t x_1
\end{align}
Note that the score associated to \eqref{eq:interp:exp} is still available because of the factor of $\sigma \zeta$, so long as $\sigma$ is invertible.

\subsection{Reducing transport costs via coupling}
\label{sec:path-length}
In the numerical experiments, we will highlight how the construction of a data-dependent coupling enables us to perform various downstream tasks.
An additional appeal is that data-dependent couplings facilitate \textit{the design of} more efficient transport than standard generation from a Gaussian, as we now show.

The bound on the transportation cost in~\eqref{eq:W2:bound} may be more tightly controlled by the construction of data-dependent couplings and their associated interpolants. In this case, we seek couplings such that $E[|\dot I_t|^2]$ is smaller with coupling than without, i.e. such that
\begin{equation}
\label{eq:bound:coupling}
\begin{aligned}
    & \int_{\R^{3d}} |\dot I_t|^2 \rho(x_0,x_1)\rho_z(z) dx_0 dx_1 dz  \\
& \le \int_{\R^{3d}} |\dot I_t|^2 \rho_0(x_0)\rho_1(x_1) \rho_z(z)  dx_0 dx_1 dz,
\end{aligned}
\end{equation}
where $\dot I_t = \dot\alpha_t x_0 + \dot \beta_t x_1  + \dot \gamma_t z$ is a function of $x_0, x_1$ and $z$. 
A simple way to design such a coupling is to consider~\eqref{eq:Gauss} with $m(x_1) = x_1$ and $C = \sigma^2 \Id$ for some $\sigma>0$, which sets the base distribution to be a noisy version of the target. 
In the case of data-decorruption (which we explore in the numerical experiments), this interpolant directly connects the corrupted conditional density and the uncorrupted density. 
If we choose $\alpha_t  = 1-t$ and $\beta_t  = t$, and set $\gamma_t=0$, then $\dot I_t = x_1-x_0$, and the left hand-side of \eqref{eq:bound:coupling} reduces to $\E[|\sigma z|^2] =  d\sigma^2$, which is less than the right hand-side given by $ 2 \E[|x_1|^2] +d\sigma^2$.

\begin{algorithm}[t]
\caption{Training \label{alg:training}}
\begin{algorithmic}
\STATE \textbf{Input:} Interpolant coefficients $\alpha_t, \beta_t$; velocity model $\hat b$; batch size $n_{\text{b}}$; 
\REPEAT
\FOR {$i = 1,\ldots, n_{\text{b}}$}
\STATE Draw $x^i_1 \sim \rho_1(x_1)$, $\zeta_i \sim \mathcal{N}(0,Id)$, $t_i \sim U(0,1)$.
\STATE Compute $x^i_0 = m(x^i_1) + \sigma \zeta^i$.
\STATE Compute $I_{t_i} = \alpha_{t_i} x^i_0 + \beta_{t_i} x^i_1$.
\ENDFOR
\STATE Compute empirical loss\\ $\hat L_b(\hat b) = n_{\text{b}}^{-1}\sum_{i=1}^{n_{\text{b}}} [| \hat b_{t_i}(I_{t_i})|^2 - 2 \dot I_{t_i} \cdot \hat b_{t_i}(I_{t_i})]$.
\STATE Take gradient step on $\hat L_b(\hat b)$ to update $\hat b$. 
\UNTIL{converged}
\STATE \textbf{Return}: Velocity $\hat{b}$.
\end{algorithmic}
\end{algorithm}

\begin{algorithm}[t]
\caption{Sampling (via forward Euler method) \label{alg:sampling}}
\begin{algorithmic}
\STATE\textbf{Input:} model $\hat b$, corrupted sample $m(x_1)$, $N\in \mathbb{N}$.
\STATE Draw noise $\zeta \sim \mathcal{N}(0,Id)$
\STATE Initialize $\hat X_0 = m(x_1) + \sigma \zeta$
\FOR{$n=0,\ldots,N-1$}
    \STATE $\hat X_{i+1}  = \hat X_i + N^{-1} \hat b_{i/N}(\hat X_i)$
\ENDFOR
\STATE \textbf{Return:} clean sample $\hat X_N$.
\end{algorithmic}
\end{algorithm}

\subsection{Learning and Sampling}
\label{sec:leran:samp}

To learn in this setup, we can evaluate the objective functions~\eqref{eq:obj} over a minibatch of $n_{\text{b}} < n$ data points $x_0^i, x_1^i$ by  using an additional $n_{\text{b}}$ samples $z_i\sim \mathsf N(0,\Id)$ and $t_i\sim U([0,1])$.
This leads to the empirical approximation $\hat{L}_b$ of $L_b$ given by
\begin{equation}
    \label{eq:obj:emp:class}
    \hat L_b(\hat b) = \frac1{n_{\text{b}}} \sum_{i=1}^{n_{\text{b}}} \left[|\hat b_{t_i}(I_{t_i})|^2 - 2\dot I_{t_i}\cdot \hat b_{t_i}(I_{t_i})\right],
\end{equation}
with a similar empirical variant for $L_z$.  
We approximate the functions $b_t(x)$ and $g_t(x)$ with neural networks and minimize these empirical objectives with stochastic gradient descent.
This leads to an approximation of the velocity $b_t(x)$ via~\eqref{eq:g:c} and
of the score via~\eqref{eq:s:c}. 

Generating data requires sampling an $X_{t=0} \sim \rho_0(x_0)$ as an initial condition to be evolved via the probability flow \gls{ode}~\eqref{eq:prob:flow} or the forward \gls{sde}~\eqref{eq:sde:f} to respectively produce a sample $X_{t=1} \sim \rho_1(x_1)$ or $X^F_{t=1} \sim \rho_1(x_1)$. 
Sampling an $x_0$ can be performed by picking data point $x_1$ either from the data set or from some online data acquisition procedure and using it in~\eqref{eq:initial:x0}, or using the assumption that one directly observes $x_0 \sim \rho_0(x_0)$ at inference time (e.g. one receives a partial image). The generated samples from either the probability flow \gls{ode} or forward \gls{sde} will be different from $x_1$, even with the choices $m(x_1) = x_1$ and $C = \sigma^2 \Id$. The probability flow \gls{ode} necessarily produces a single sample of $x_1$ for each $x_0$, while the \gls{sde} produces a collection of samples whose spread can be controlled by the diffusion coefficient $\epsilon_t$. Algorithms \ref{alg:training} and \ref{alg:sampling} depict these training and sampling procedures, respectively.

\section{Numerical experiments}
\label{sec:experiments}
We now explore the interpolants with data-dependent couplings on conditional image generation tasks; we find that the framework is straightforward to scale to high resolution images directly in pixel space.
\subsection{In-painting}
\label{sec:inp}

We consider an in-painting task, whereby $x_1 \in \R^{C\times W \times H}$ denotes an image with~$C$ channels, width~$W$, and height~$H$.
Given a pre-specified mask, the goal is to fill the pixels in the masked region with new values that are consistent with the entirety of the image.
We set the conditioning variable $\xi\in \{0, 1\}^{C\times W \times H}$ and additionally provide the model with any potential class labels.
For simplicity, the mask takes the same value for all channels in a given spatial location in the image.
We define the base density by the relation $x_0 = \xi \circ x_1 + (1 - \xi) \circ \zeta$, where $\circ$ denotes the Hadamard (elementwise) product and $\zeta\in\R^{C\times W \times H}, \zeta \sim \mathsf N(0, \Id)$ denotes random noise used to initialize the pixels within the masked region (separate  noise for each channel).
During training, the mask is drawn randomly by tiling the image into~$64$ tiles; each tile is selected to enter the mask with probability~$p=0.3$.
In our experiments, we set $\rho_1(x_1)$ to correspond to ImageNet (either $256$ or $512$).
This corresponds to using $\rho(x_0,x_1|\xi) = \rho_1(x_1) \rho_0(x_0|x_1,\xi)$. The model
sees the mask; we note that we do not need to 
additionally input the partial image as extra conditioning because it is present, uncorrupted, in $x_t$ for each $t$ because the values are present in $x_0$ and $x_1$.
In the interpolant~\eqref{eq:interp:exp}, we set $\alpha_t  = t$ and $\beta_t  = 1-t$.
In this setup, the velocity field $b_t(x, \xi)$ is such that $b_t(x,\xi)=0$ except in the masked regions.
This follows because $\xi \circ I_t = \xi \circ x_1$ for every $t$, i.e., the unmasked pixels in $I_t$ are always those of $x_1$ for which $\dot I_t=0$.
To take this structural information into account, we can build this property into our neural network model, and mask the output of the approximate velocity field to enforce that the unmasked pixels remain fixed. We note that this method does not necessitate any inference time corrections, such as the replacement method or MCMC.
\begin{table}[b]
\vspace{-0.05cm}
\centering
\caption{\textbf{FID for Inpainting Task.} FID comparison between under two paradigms: a baseline, where $\rho_0$ is a Gaussian with independent coupling to $\rho_1$, and our data-dependent coupling detailed in Section \ref{sec:inp}.}
\label{tab:inpaint}
\small
\begin{tabular}{ll}
\toprule
\textbf{Model} & \textbf{FID-50k} \\
\midrule
Uncoupled Interpolant (Baseline) & 1.35 \\
Dependent Coupling (\textbf{Ours}) & \textbf{1.13} \\
\bottomrule
\end{tabular}
\end{table}

\begin{figure*}[ht]
    \centering    
    \includegraphics[width=0.7\linewidth]{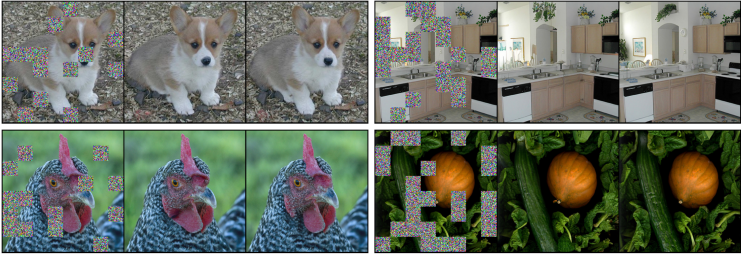}
    \vspace{-5pt}
    \raisebox{8pt}{\rule{0.8\linewidth}{1pt}} 
    \includegraphics[width=0.7\linewidth]{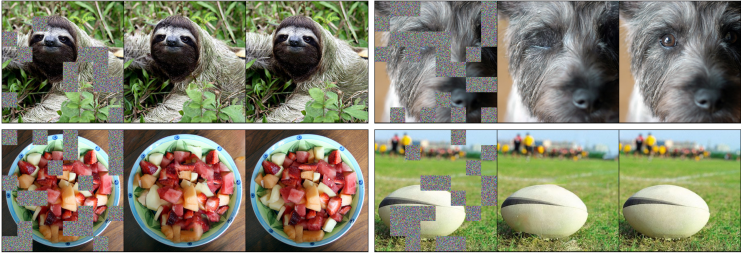}
    \caption{\textbf{Image inpainting:  ImageNet-$256\times 256$ and ImageNet-$512\times 512$.} \textit{Top panels}: Six examples of image in-filling at resolution $256\times 256$, where the left columns display masked images, the center corresponds to in-filled model samples, and the right shows full reference images. The aims are not to recover the precise content of the reference image, but instead, to provide a conditionally valid in-filling.
    \textit{Bottom panels}: Four examples at resolution $512\times 512$.}
    \label{fig:mask}
\end{figure*}

\paragraph{Results.}
For implementation, we parameterize $b_t(x, \xi)$ using the basic U-Net architecture from \cite{ho2020denoising}, where $\xi$ is given to the model as appended channels of the image $x$.
Additional specific experimental details may be found in  \cref{appsec:experimental}.
Samples are shown in \Cref{fig:mask}, as well as 
\cref{fig:front_page}. FIDs are reported in 
\cref{tab:inpaint}. As discussed, the missing areas of the image are defined at time zero as independent normal random variables, depicted as colorful static in the images. In each image triple, the left panel is the base distribution sample $x_0$, the middle is the model sample of $X_{t=1}$ obtained by integrating the probability flow \gls{ode} \eqref{eq:prob:flow}, and the right panel is the ground truth. The generated textures, though different from the full sample, correspond to realistic samples from the conditional densities given the observed content. This is an advantage of probabilistic generative models such as ours over models optimized to fit a mean-square error to a ground truth image.

\subsection{Super-resolution on Imagenet}
\label{sec:super}

\begin{table}
\centering
\caption{\textbf{FID-50k for Super-resolution, 64x64 to 256x256.}
FIDs for baselines taken from 
\cite{saharia2022image, ho2022cascaded, liu20232}.
}
\label{tab:super}
\small
\begin{tabular}{lll}
\toprule
\textbf{Model} & \textbf{Train} & \textbf{Valid}\\
\midrule
Improved DDPM \citep{nichol2021improved} & 12.26 & --\\
SR3 \citep{saharia2022image} & 11.30 &  5.20 \\
ADM \citep{dhariwal2021diffusion} & 7.49 & 3.10 \\
Cascaded Diffusion \citep{ho2022cascaded} & 4.88  & 4.63\\
$\text{I}^2$SB \citep{liu20232} & -- & 2.70 \\
Dependent Coupling (\textbf{Ours}) & \textbf{2.13} & \textbf{2.05} \\
\bottomrule
\end{tabular}
\end{table}
We now consider image super-resolution, in which we would like to produce an image with the same content as a given image but at higher resolution. 
\begin{figure*}[h!]
    \centering
    \includegraphics[width=0.40\linewidth]
    {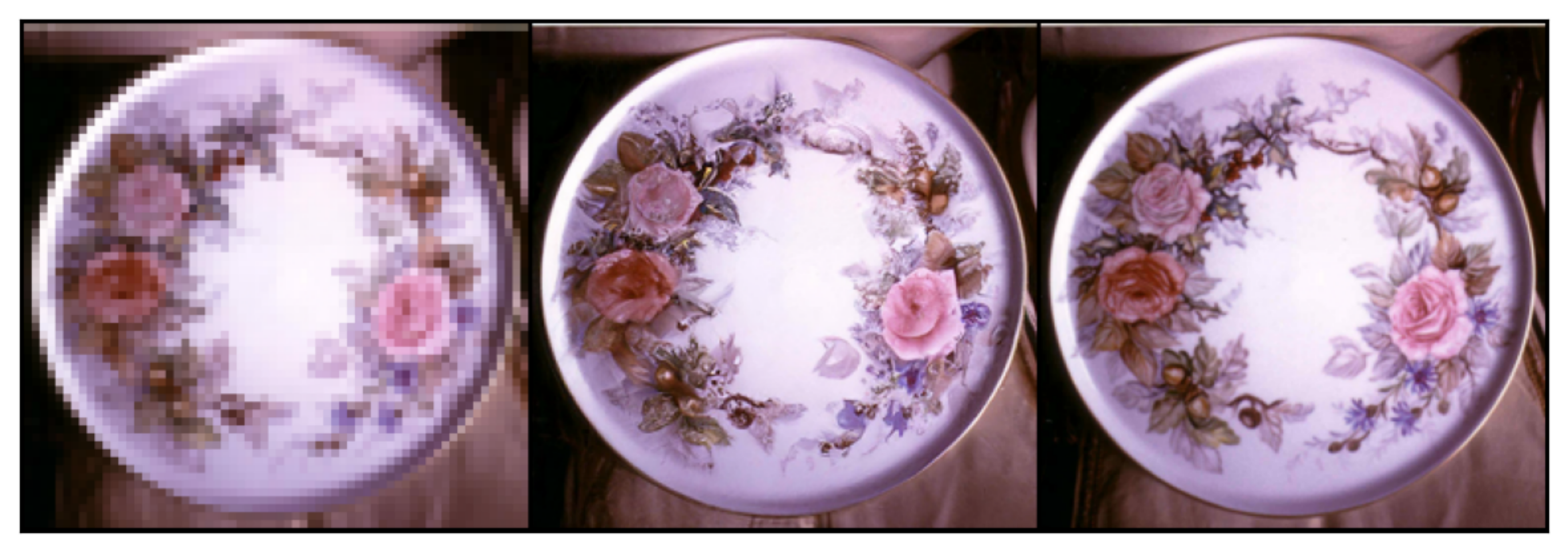}
    \includegraphics[width=0.40\linewidth]
    {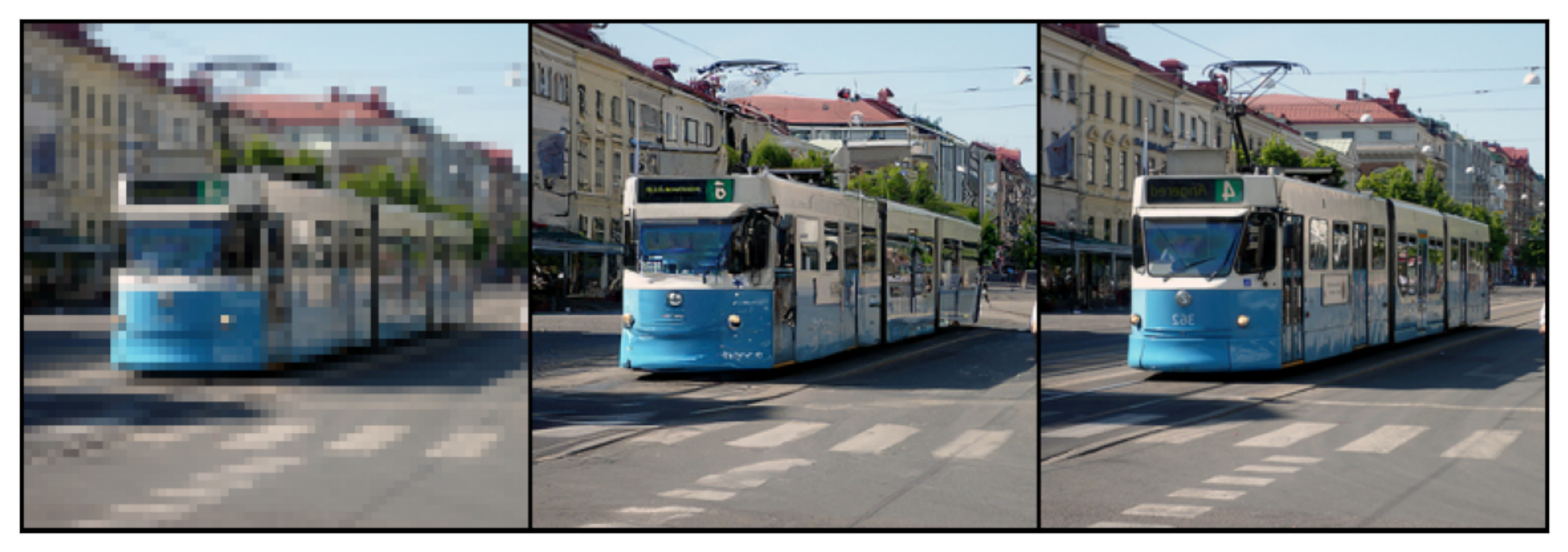}\\
    \includegraphics[width=0.40\linewidth]
    {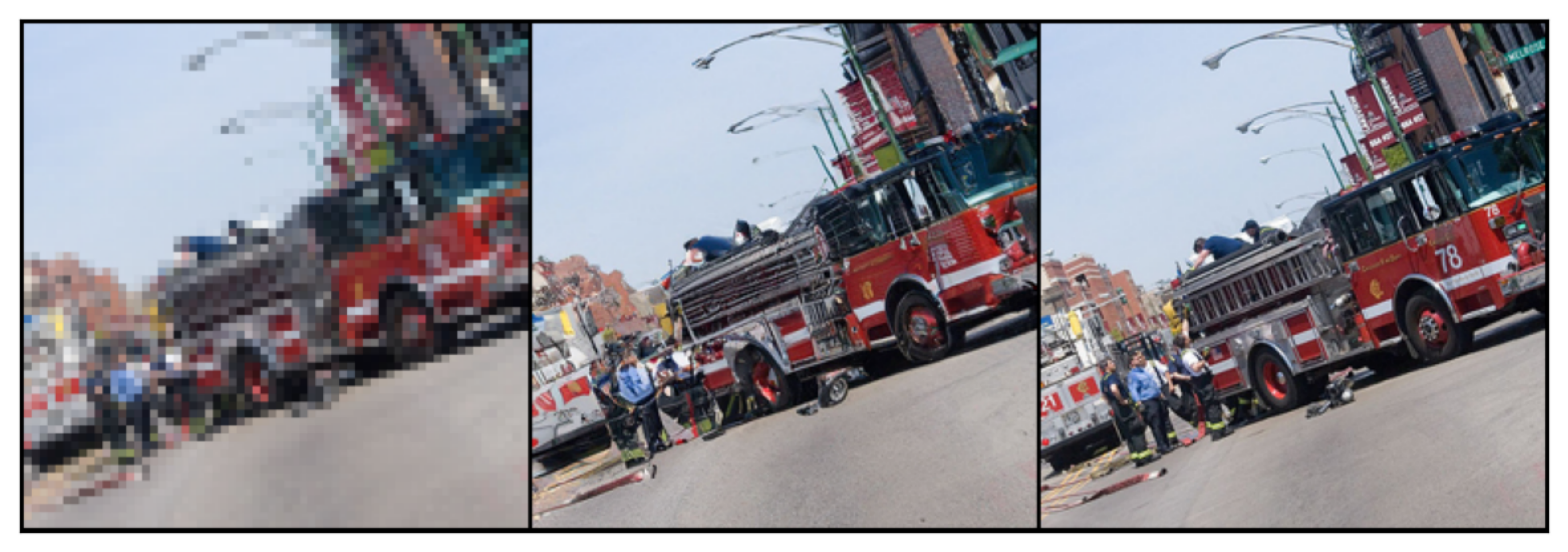}
    \includegraphics[width=0.40\linewidth]
    {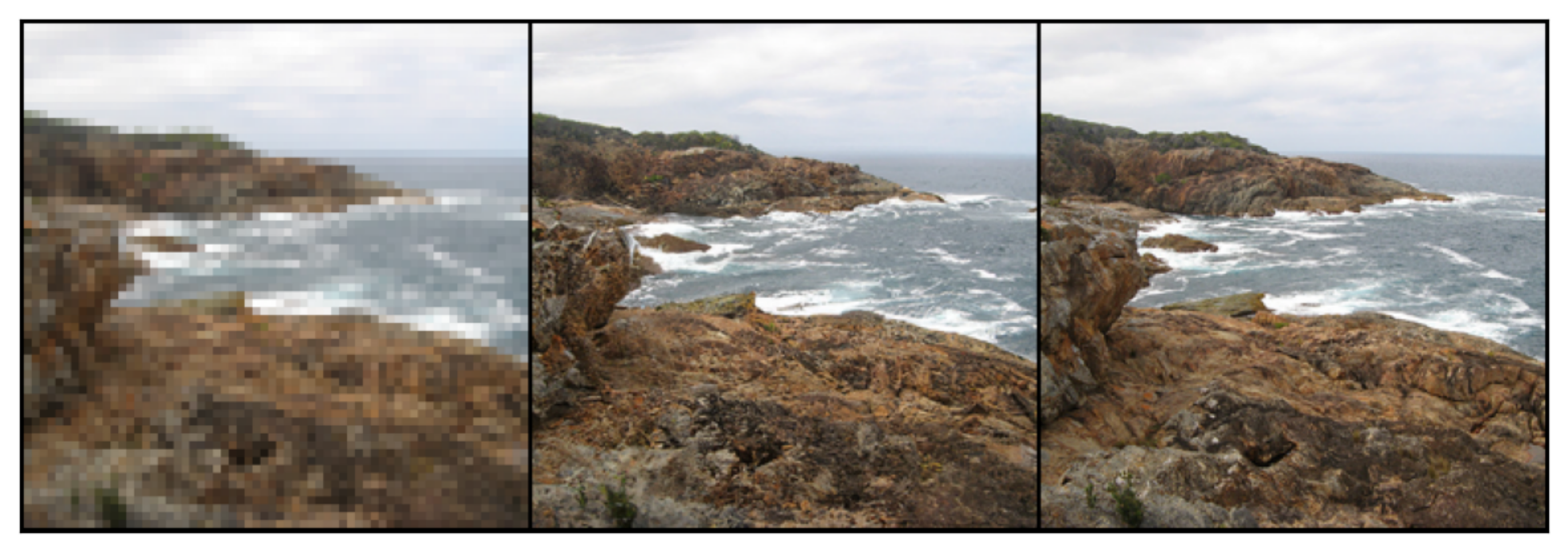}\\
    \includegraphics[width=0.397\linewidth]
    {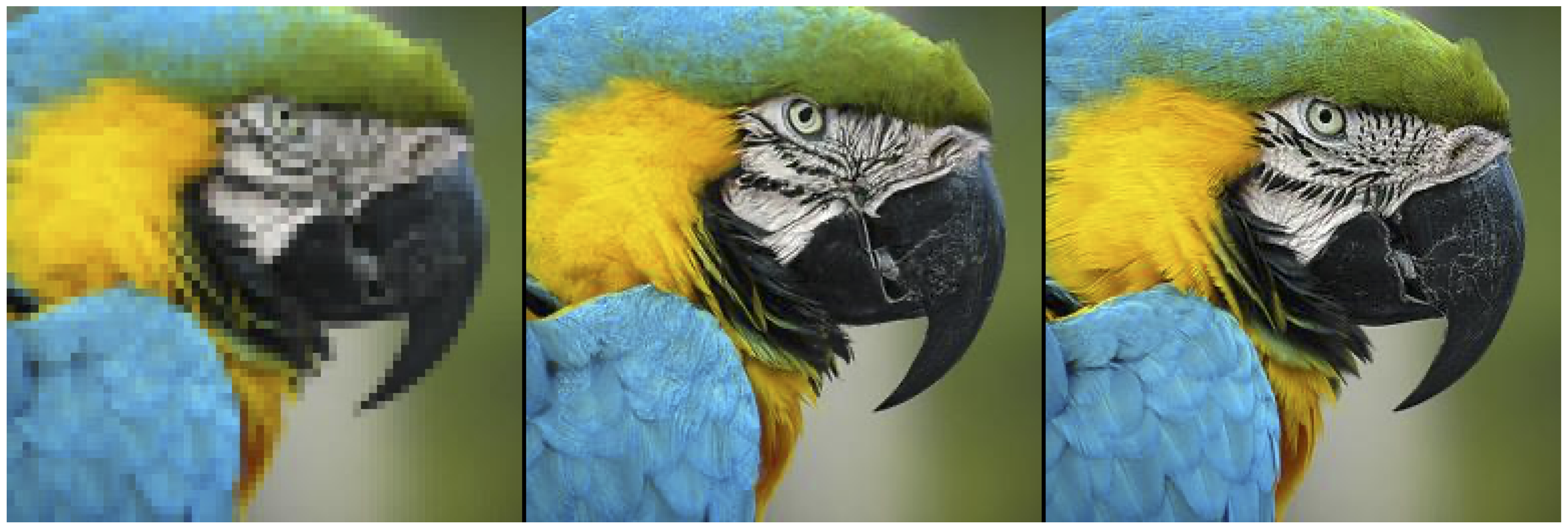}
    \includegraphics[width=0.397\linewidth]
    {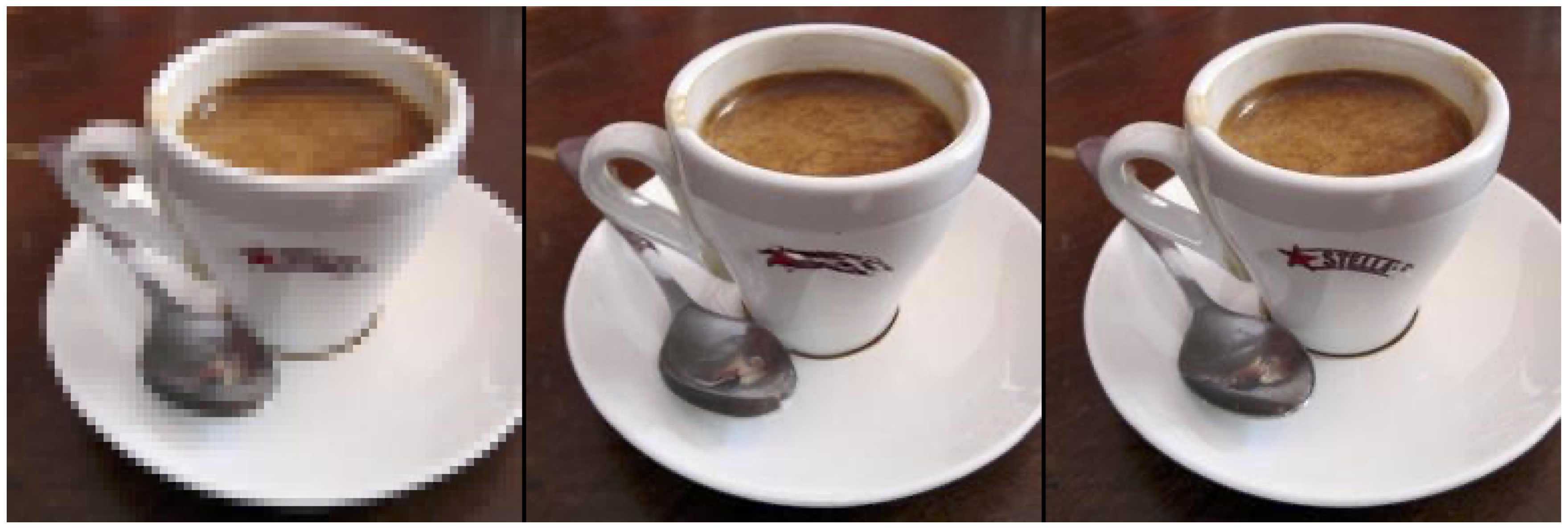}

    \caption{\textbf{Super-resolution:}
    \textit{Top four rows}: Super-resolved images from resolution $64\times 64 \mapsto 256\times 256$, where the left-most image is the lower resolution version, the middle is the model output, and the right is the ground truth. Examples for $256\times 256 \mapsto 512\times 512$ are given in \cref{fig:super512}.}
    \label{fig:super}
\end{figure*}
To this end, we let $x_1 \in \R^{C\times W \times H}$ correspond to a high-resolution image, as in~\cref{sec:inp}.
We denote by $\mathcal{D}:\R^{C\times W\times H} \rightarrow \R^{C\times W_{\text{low}} \times H_{\text{low}}}$ and $\mathcal{U}: \R^{C\times W_{\text{low}} \times H_{\text{low}}} \rightarrow \R^{C\times W \times H}$ image downsampling and upsampling operations, where $W_{\text{low}}$ and $H_{\text{low}}$ denote the width and height of a low-resolution image.
To define the base density, we then set $x_0 = \mathcal{U}\left(\mathcal{D}\left(x_1\right)\right) + \sigma \zeta$ with $\zeta \in \R^{C\times W \times H}$, $\zeta\sim\mathsf{N}(0, \Id)$, and $\sigma>0$.
Defining $x_0$ in this way frames the transport problem such that each starting pixel is proximal to its intended target.
Notice in particular that, with $\sigma = 0$, each $x_0$ would correspond to a lower-dimensional sample embedded in a higher-dimensional space, and the corresponding distribution would be concentrated on a lower-dimensional manifold.
Working with $\sigma>0$ alleviates the associated singularities by adding a small amount of Gaussian noise to smooth the base density so it is well-defined over the entire higher-dimensional ambient space.
In addition, we give the model access to the low-resolution image at all times; this problem setting then corresponds to using $\rho(x_0, x_1 | \xi) = \rho_1(x_1)\rho_0(x_0|x_1, \xi)$ with $\xi = \mathcal{U}\left(\mathcal{D}\left(x_1\right)\right)$.
In the experiments, we set $\rho_1$ to correspond to ImageNet ($256$ or $512$), following prior work~\citep{saharia2022image, ho2022cascaded}.
\paragraph{Results.} 
Similarly to the previous experiment, we append the upsampled low-resolution images $\xi$ to the channel dimension of the input $x$ of the velocity model, and likewise include the ImageNet class labels. Samples are displayed in \cref{fig:super}, as well as \cref{fig:front_page}. 
Similar in layout to the previous experiment, the left panel of each triplet is the low-resolution image, the middle panel is the model sample $X_{t=1}$, and the right panel is the high-resolution image. The differences are easiest to see when zoomed-in. While the increased resolution of the model sample is very noticeable for 64 to 256, the differences even in ground truth images between 256 and 512 are more subtle. We also display FIDs
for the 64x64 to 256x256 task, which has been studied in other works, in \cref{tab:super}.

\section{Discussion, challenges, and future work}
\label{sec:disc}
In this work, we introduced a general framework for constructing data-dependent couplings between base and target densities within the stochastic interpolant formalism.
We provide some suggestions for specific forms of data-dependent coupling, such as choosing for $\rho_0$ a Gaussian distribution with mean and covariance adapted to samples from the target, and showed how they can be used in practical problem settings such as image inpainting and super-resolution.
There are many interesting generative modeling problems that stand to benefit from the incorporation of data-dependent structure.
In the sciences, one potential application is in molecule generation, where we can imagine using data-dependent base distributions to fix a chemical backbone and vary functional groups.
The dependency and conditioning structure needed to accomplish a task like this
is similar to image inpainting.
In machine learning, one potential application is in correcting autoencoding errors produced by an architecture such as a variational autoencoder~\citep{kingma_auto-encoding_2013}, where we could take the target density to be inputs to the autoencoder and the base density to be the output of the autoencoder.

\section*{Acknowledgements}
We thank Raghav Singhal for insightful discussions.
MG and RR are partly supported by the NIH/NHLBI Award R01HL148248, NSF Award 1922658 NRT-HDR: FUTURE Foundations, Translation, and Responsibility for Data Science, NSF CAREER Award 2145542, ONR N00014-23-1-2634, and Apple.
MSA and NMB are funded by the ONR project entitled Mathematical Foundation and Scientific Applications of Machine Learning.
EVE is supported by the National Science Foundation under Awards DMR-1420073, DMS-2012510, and DMS-2134216, by the Simons Collaboration on Wave Turbulence, Grant No. 617006, and by a Vannevar Bush Faculty Fellowship.
    
\section*{Impact Statement}
While this paper presents work whose goal is to advance the field of machine learning, and there are many potential societal consequences of our work, we wish to highlight that generative models, as they are currently used, pose the risk of perpetuating harmful biases and stereotypes.

\bibliography{icml2024/icml2024_cameraready}

\begin{thebibliography}{48}
\providecommand{\natexlab}[1]{#1}
\providecommand{\url}[1]{\texttt{#1}}
\expandafter\ifx\csname urlstyle\endcsname\relax
  \providecommand{\doi}[1]{doi: #1}\else
  \providecommand{\doi}{doi: \begingroup \urlstyle{rm}\Url}\fi

\bibitem[Albergo \& Vanden-Eijnden(2022)Albergo and Vanden-Eijnden]{albergo2022building}
Albergo, M.~S. and Vanden-Eijnden, E.
\newblock Building normalizing flows with stochastic interpolants.
\newblock \emph{arXiv preprint arXiv:2209.15571}, 2022.

\bibitem[Albergo et~al.(2023)Albergo, Boffi, and Vanden-Eijnden]{albergo2023stochastic}
Albergo, M.~S., Boffi, N.~M., and Vanden-Eijnden, E.
\newblock Stochastic interpolants: A unifying framework for flows and diffusions.
\newblock \emph{arXiv preprint arXiv:2303.08797}, 2023.

\bibitem[Benamou \& Brenier(2000)Benamou and Brenier]{benamou2000}
Benamou, J.-D. and Brenier, Y.
\newblock A computational fluid mechanics solution to the monge-kantorovich mass transfer problem.
\newblock \emph{Numerische Mathematik}, 84\penalty0 (3):\penalty0 375--393, 2000.
\newblock \doi{10.1007/s002110050002}.
\newblock URL \url{https://doi.org/10.1007/s002110050002}.

\bibitem[Chen \& Lipman(2023)Chen and Lipman]{chen2023riemannian}
Chen, R.~T. and Lipman, Y.
\newblock Riemannian flow matching on general geometries.
\newblock \emph{arXiv preprint arXiv:2302.03660}, 2023.

\bibitem[Chen(2018)]{torchdiffeq}
Chen, R. T.~Q.
\newblock torchdiffeq, 2018.
\newblock URL \url{https://github.com/rtqichen/torchdiffeq}.

\bibitem[Cuturi(2013)]{cuturi2013sinkhorn}
Cuturi, M.
\newblock Sinkhorn distances: Lightspeed computation of optimal transport.
\newblock \emph{Advances in neural information processing systems}, 26, 2013.

\bibitem[Dao et~al.(2023)Dao, Phung, Nguyen, and Tran]{dao2023flow}
Dao, Q., Phung, H., Nguyen, B., and Tran, A.
\newblock Flow matching in latent space.
\newblock \emph{arXiv preprint arXiv:2307.08698}, 2023.

\bibitem[De~Bortoli et~al.(2021)De~Bortoli, Thornton, Heng, and Doucet]{bortoli2021}
De~Bortoli, V., Thornton, J., Heng, J., and Doucet, A.
\newblock Diffusion schr\"{o}dinger bridge with applications to score-based generative modeling.
\newblock In Ranzato, M., Beygelzimer, A., Dauphin, Y., Liang, P., and Vaughan, J.~W. (eds.), \emph{Advances in Neural Information Processing Systems}, volume~34, pp.\  17695--17709. Curran Associates, Inc., 2021.
\newblock URL \url{https://proceedings.neurips.cc/paper_files/paper/2021/file/940392f5f32a7ade1cc201767cf83e31-Paper.pdf}.

\bibitem[Dhariwal \& Nichol(2021)Dhariwal and Nichol]{dhariwal2021diffusion}
Dhariwal, P. and Nichol, A.
\newblock Diffusion models beat gans on image synthesis.
\newblock \emph{Advances in neural information processing systems}, 34:\penalty0 8780--8794, 2021.

\bibitem[Dinh et~al.(2017)Dinh, {Sohl-Dickstein}, and Bengio]{dinh_density_2017}
Dinh, L., {Sohl-Dickstein}, J., and Bengio, S.
\newblock Density {{Estimation Using Real NVP}}.
\newblock In \emph{International Conference on Learning Representations}, pp.\ ~32, 2017.

\bibitem[Durkan et~al.(2019)Durkan, Bekasov, Murray, and Papamakarios]{durkan2019}
Durkan, C., Bekasov, A., Murray, I., and Papamakarios, G.
\newblock Neural spline flows.
\newblock In Wallach, H., Larochelle, H., Beygelzimer, A., d\textquotesingle Alch\'{e}-Buc, F., Fox, E., and Garnett, R. (eds.), \emph{Advances in Neural Information Processing Systems}, volume~32. Curran Associates, Inc., 2019.
\newblock URL \url{https://proceedings.neurips.cc/paper/2019/file/7ac71d433f282034e088473244df8c02-Paper.pdf}.

\bibitem[Ho \& Salimans(2022)Ho and Salimans]{ho2022classifier}
Ho, J. and Salimans, T.
\newblock Classifier-free diffusion guidance.
\newblock \emph{arXiv preprint arXiv:2207.12598}, 2022.

\bibitem[Ho et~al.(2020{\natexlab{a}})Ho, Jain, and Abbeel]{ho2020}
Ho, J., Jain, A., and Abbeel, P.
\newblock Denoising diffusion probabilistic models.
\newblock In Larochelle, H., Ranzato, M., Hadsell, R., Balcan, M., and Lin, H. (eds.), \emph{Advances in Neural Information Processing Systems}, volume~33, pp.\  6840--6851. Curran Associates, Inc., 2020{\natexlab{a}}.
\newblock URL \url{https://proceedings.neurips.cc/paper/2020/file/4c5bcfec8584af0d967f1ab10179ca4b-Paper.pdf}.

\bibitem[Ho et~al.(2020{\natexlab{b}})Ho, Jain, and Abbeel]{ho2020denoising}
Ho, J., Jain, A., and Abbeel, P.
\newblock Denoising diffusion probabilistic models.
\newblock \emph{Advances in neural information processing systems}, 33:\penalty0 6840--6851, 2020{\natexlab{b}}.

\bibitem[Ho et~al.(2022{\natexlab{a}})Ho, Saharia, Chan, Fleet, Norouzi, and Salimans]{ho2022cascaded}
Ho, J., Saharia, C., Chan, W., Fleet, D.~J., Norouzi, M., and Salimans, T.
\newblock Cascaded diffusion models for high fidelity image generation.
\newblock \emph{The Journal of Machine Learning Research}, 23\penalty0 (1):\penalty0 2249--2281, 2022{\natexlab{a}}.

\bibitem[Ho et~al.(2022{\natexlab{b}})Ho, Salimans, Gritsenko, Chan, Norouzi, and Fleet]{ho2022video}
Ho, J., Salimans, T., Gritsenko, A., Chan, W., Norouzi, M., and Fleet, D.~J.
\newblock Video diffusion models.
\newblock \emph{arXiv:2204.03458}, 2022{\natexlab{b}}.

\bibitem[Hu et~al.(2023)Hu, Zhang, Tang, Mettes, Zhao, and Snoek]{hu2023latent}
Hu, V.~T., Zhang, D.~W., Tang, M., Mettes, P., Zhao, D., and Snoek, C.~G.
\newblock Latent space editing in transformer-based flow matching.
\newblock In \emph{ICML Workshop on New Frontiers in Learning, Control, and Dynamical Systems}, 2023.

\bibitem[Huang et~al.(2016)Huang, Sun, Liu, Sedra, and Weinberger]{huang_deep_2016}
Huang, G., Sun, Y., Liu, Z., Sedra, D., and Weinberger, K.
\newblock Deep {{Networks}} with {{Stochastic Depth}}.
\newblock \emph{arXiv:1603.09382 [cs]}, July 2016.

\bibitem[Kingma \& Ba(2014)Kingma and Ba]{kingma2014adam}
Kingma, D.~P. and Ba, J.
\newblock Adam: A method for stochastic optimization.
\newblock \emph{arXiv preprint arXiv:1412.6980}, 2014.

\bibitem[Kingma \& Welling(2013)Kingma and Welling]{kingma_auto-encoding_2013}
Kingma, D.~P. and Welling, M.
\newblock Auto-{{Encoding Variational Bayes}}.
\newblock \emph{arXiv [Preprint]}, 0, 2013.
\newblock URL \url{https://arxiv.org/1312.6114v10}.

\bibitem[Klein et~al.(2023)Klein, Krämer, and Noé]{klein2023equivariant}
Klein, L., Krämer, A., and Noé, F.
\newblock Equivariant flow matching, 2023.

\bibitem[Lee et~al.(2023)Lee, Kim, and Ye]{lee2023minimizing}
Lee, S., Kim, B., and Ye, J.~C.
\newblock Minimizing trajectory curvature of ode-based generative models.
\newblock \emph{arXiv preprint arXiv:2301.12003}, 2023.

\bibitem[Lipman et~al.(2022{\natexlab{a}})Lipman, Chen, Ben-Hamu, Nickel, and Le]{lipman2022flow}
Lipman, Y., Chen, R.~T., Ben-Hamu, H., Nickel, M., and Le, M.
\newblock Flow matching for generative modeling.
\newblock \emph{arXiv preprint arXiv:2210.02747}, 2022{\natexlab{a}}.

\bibitem[Lipman et~al.(2022{\natexlab{b}})Lipman, Chen, Ben-Hamu, Nickel, and Le]{lipman2022}
Lipman, Y., Chen, R. T.~Q., Ben-Hamu, H., Nickel, M., and Le, M.
\newblock Flow matching for generative modeling, 2022{\natexlab{b}}.
\newblock URL \url{https://arxiv.org/abs/2210.02747}.

\bibitem[Liu et~al.(2023{\natexlab{a}})Liu, Vahdat, Huang, Theodorou, Nie, and Anandkumar]{liu20232}
Liu, G.-H., Vahdat, A., Huang, D.-A., Theodorou, E.~A., Nie, W., and Anandkumar, A.
\newblock $\text{I}^2$sb: Image-to-image schr$\backslash$" odinger bridge.
\newblock \emph{arXiv preprint arXiv:2302.05872}, 2023{\natexlab{a}}.

\bibitem[Liu(2022)]{liu2022-ot}
Liu, Q.
\newblock Rectified flow: A marginal preserving approach to optimal transport, 2022.
\newblock URL \url{https://arxiv.org/abs/2209.14577}.

\bibitem[Liu et~al.(2022{\natexlab{a}})Liu, Gong, and Liu]{liu2022}
Liu, X., Gong, C., and Liu, Q.
\newblock Flow straight and fast: Learning to generate and transfer data with rectified flow, 2022{\natexlab{a}}.
\newblock URL \url{https://arxiv.org/abs/2209.03003}.

\bibitem[Liu et~al.(2022{\natexlab{b}})Liu, Gong, and Liu]{liu2022flow}
Liu, X., Gong, C., and Liu, Q.
\newblock Flow straight and fast: Learning to generate and transfer data with rectified flow.
\newblock \emph{arXiv preprint arXiv:2209.03003}, 2022{\natexlab{b}}.

\bibitem[Liu et~al.(2023{\natexlab{b}})Liu, Zhang, Ma, Peng, and Liu]{liu2023instaflow}
Liu, X., Zhang, X., Ma, J., Peng, J., and Liu, Q.
\newblock Instaflow: One step is enough for high-quality diffusion-based text-to-image generation.
\newblock \emph{arXiv preprint arXiv:2309.06380}, 2023{\natexlab{b}}.

\bibitem[Nichol \& Dhariwal(2021)Nichol and Dhariwal]{nichol2021improved}
Nichol, A.~Q. and Dhariwal, P.
\newblock Improved denoising diffusion probabilistic models.
\newblock In \emph{International Conference on Machine Learning}, pp.\  8162--8171. PMLR, 2021.

\bibitem[Pooladian et~al.(2023)Pooladian, Ben-Hamu, Domingo-Enrich, Amos, Lipman, and Chen]{pooladian2023multisample}
Pooladian, A.-A., Ben-Hamu, H., Domingo-Enrich, C., Amos, B., Lipman, Y., and Chen, R.
\newblock Multisample flow matching: Straightening flows with minibatch couplings.
\newblock \emph{arXiv preprint arXiv:2304.14772}, 2023.

\bibitem[Rezende \& Mohamed(2015)Rezende and Mohamed]{rezende_variational_2015}
Rezende, D. and Mohamed, S.
\newblock Variational {{Inference}} with {{Normalizing Flows}}.
\newblock In \emph{International {{Conference}} on {{Machine Learning}}}, pp.\  1530--1538. {PMLR}, June 2015.

\bibitem[Saharia et~al.(2022)Saharia, Ho, Chan, Salimans, Fleet, and Norouzi]{saharia2022image}
Saharia, C., Ho, J., Chan, W., Salimans, T., Fleet, D.~J., and Norouzi, M.
\newblock Image super-resolution via iterative refinement.
\newblock \emph{IEEE Transactions on Pattern Analysis and Machine Intelligence}, 45\penalty0 (4):\penalty0 4713--4726, 2022.

\bibitem[Schneuing et~al.(2022)Schneuing, Du, Harris, Jamasb, Igashov, Du, Blundell, Li{\'o}, Gomes, Welling, et~al.]{schneuing2022structure}
Schneuing, A., Du, Y., Harris, C., Jamasb, A., Igashov, I., Du, W., Blundell, T., Li{\'o}, P., Gomes, C., Welling, M., et~al.
\newblock Structure-based drug design with equivariant diffusion models.
\newblock \emph{arXiv preprint arXiv:2210.13695}, 2022.

\bibitem[Shi et~al.(2022)Shi, Bortoli, Deligiannidis, and Doucet]{shi2022conditional}
Shi, Y., Bortoli, V.~D., Deligiannidis, G., and Doucet, A.
\newblock Conditional simulation using diffusion schr\"odinger bridges.
\newblock In \emph{The 38th Conference on Uncertainty in Artificial Intelligence}, 2022.
\newblock URL \url{https://openreview.net/forum?id=H9Lu6P8sqec}.

\bibitem[Shi et~al.(2023)Shi, Bortoli, Campbell, and Doucet]{shi2023diffusion}
Shi, Y., Bortoli, V.~D., Campbell, A., and Doucet, A.
\newblock Diffusion schr\"odinger bridge matching, 2023.

\bibitem[Singhal et~al.(2023)Singhal, Goldstein, and Ranganath]{singhal2023diffuse}
Singhal, R., Goldstein, M., and Ranganath, R.
\newblock Where to diffuse, how to diffuse, and how to get back: Automated learning for multivariate diffusions.
\newblock In \emph{The Eleventh International Conference on Learning Representations}, 2023.

\bibitem[Sohl-Dickstein et~al.(2015)Sohl-Dickstein, Weiss, Maheswaranathan, and Ganguli]{sohl2015deep}
Sohl-Dickstein, J., Weiss, E., Maheswaranathan, N., and Ganguli, S.
\newblock Deep unsupervised learning using nonequilibrium thermodynamics.
\newblock In \emph{International conference on machine learning}, pp.\  2256--2265. PMLR, 2015.

\bibitem[Somnath et~al.(2023)Somnath, Pariset, Hsieh, Martinez, Krause, and Bunne]{somnath2023aligned}
Somnath, V.~R., Pariset, M., Hsieh, Y.-P., Martinez, M.~R., Krause, A., and Bunne, C.
\newblock Aligned diffusion schr\"odinger bridges.
\newblock In \emph{The 39th Conference on Uncertainty in Artificial Intelligence}, 2023.
\newblock URL \url{https://openreview.net/forum?id=BkWFJN7_bQ}.

\bibitem[Song \& Ermon(2020)Song and Ermon]{song2020improved}
Song, Y. and Ermon, S.
\newblock Improved techniques for training score-based generative models.
\newblock \emph{Advances in neural information processing systems}, 33:\penalty0 12438--12448, 2020.

\bibitem[Song et~al.(2020)Song, Sohl-Dickstein, Kingma, Kumar, Ermon, and Poole]{song2020score}
Song, Y., Sohl-Dickstein, J., Kingma, D.~P., Kumar, A., Ermon, S., and Poole, B.
\newblock Score-based generative modeling through stochastic differential equations.
\newblock \emph{arXiv preprint arXiv:2011.13456}, 2020.

\bibitem[Song et~al.(2021{\natexlab{a}})Song, Durkan, Murray, and Ermon]{song2021mle}
Song, Y., Durkan, C., Murray, I., and Ermon, S.
\newblock Maximum likelihood training of score-based diffusion models.
\newblock In Ranzato, M., Beygelzimer, A., Dauphin, Y., Liang, P., and Vaughan, J.~W. (eds.), \emph{Advances in Neural Information Processing Systems}, volume~34, pp.\  1415--1428. Curran Associates, Inc., 2021{\natexlab{a}}.
\newblock URL \url{https://proceedings.neurips.cc/paper/2021/file/0a9fdbb17feb6ccb7ec405cfb85222c4-Paper.pdf}.

\bibitem[Song et~al.(2021{\natexlab{b}})Song, {Sohl-Dickstein}, Kingma, Kumar, Ermon, and Poole]{song_scorebased_2021}
Song, Y., {Sohl-Dickstein}, J., Kingma, D.~P., Kumar, A., Ermon, S., and Poole, B.
\newblock Score-based generative modeling through stochastic differential equations.
\newblock In \emph{International Conference on Learning Representations}, 2021{\natexlab{b}}.

\bibitem[Tabak \& Turner(2013)Tabak and Turner]{tabak2013}
Tabak, E.~G. and Turner, C.~V.
\newblock A family of nonparametric density estimation algorithms.
\newblock \emph{Communications on Pure and Applied Mathematics}, 66\penalty0 (2):\penalty0 145--164, 2013.
\newblock \doi{https://doi.org/10.1002/cpa.21423}.
\newblock URL \url{https://onlinelibrary.wiley.com/doi/abs/10.1002/cpa.21423}.

\bibitem[Tabak \& {Vanden-Eijnden}(2010)Tabak and {Vanden-Eijnden}]{tabak_density_2010}
Tabak, E.~G. and {Vanden-Eijnden}, E.
\newblock Density estimation by dual ascent of the log-likelihood.
\newblock \emph{Communications in Mathematical Sciences}, 8\penalty0 (1):\penalty0 217--233, 2010.
\newblock ISSN 15396746, 19450796.
\newblock \doi{10.4310/CMS.2010.v8.n1.a11}.

\bibitem[Tong et~al.(2023)Tong, Malkin, Huguet, Zhang, Rector-Brooks, Fatras, Wolf, and Bengio]{tong2023improving}
Tong, A., Malkin, N., Huguet, G., Zhang, Y., Rector-Brooks, J., Fatras, K., Wolf, G., and Bengio, Y.
\newblock Improving and generalizing flow-based generative models with minibatch optimal transport.
\newblock In \emph{ICML Workshop on New Frontiers in Learning, Control, and Dynamical Systems}, 2023.

\bibitem[Trippe et~al.(2022)Trippe, Yim, Tischer, Baker, Broderick, Barzilay, and Jaakkola]{trippe2022diffusion}
Trippe, B.~L., Yim, J., Tischer, D., Baker, D., Broderick, T., Barzilay, R., and Jaakkola, T.
\newblock Diffusion probabilistic modeling of protein backbones in 3d for the motif-scaffolding problem.
\newblock \emph{arXiv preprint arXiv:2206.04119}, 2022.

\bibitem[Wu et~al.(2023)Wu, Trippe, Naesseth, Blei, and Cunningham]{wu2023practical}
Wu, L., Trippe, B.~L., Naesseth, C.~A., Blei, D.~M., and Cunningham, J.~P.
\newblock Practical and asymptotically exact conditional sampling in diffusion models.
\newblock \emph{arXiv preprint arXiv:2306.17775}, 2023.

\end{thebibliography}
\bibliographystyle{icml2024}

\clearpage
\appendix
\onecolumn

\section{Omitted proofs with conditioning variables incorporated}
\label{sec:proofs}

In this Appendix we give the proofs of Theorem~\ref{th:pdf0} and Corollary~\ref{th:gen:mod0} in a more general setup in which we incorporate conditioning variables in the definition of the stochastic interpolant. 

To this end, suppose that each data point $x_1^i\in \R^d$ in the data set comes with a label $\xi_i\in D$, such as a discrete class or a continuous embedding like a text caption, and let us assume that 
this data set comes from a 
\gls{pdf} decomposed as $\rho_1(x_1|\xi) \eta(\xi)$, where $\rho_1(x_1|\xi)$ is the density of the data $x_1$ conditioned on their label~$\xi$, and $\eta(\xi)$ is the density of the label.
In the following, we will somewhat abuse notation and use $\eta(\xi)$ even when $\xi$ is discrete (in which case, $\eta(\xi)$ is a sum of Dirac measures); we will however assume that $\rho_1(x_1|\xi)$ is a proper density.
In this setup we can generalize Definition~\ref{def:cond:interp} as
\begin{definition}[Stochastic interpolant with coupling and conditioning]
\label{def:cond:interp:c}
The  stochastic interpolant $I_t$ is the stochastic process defined as
\begin{equation}
    \label{eq:stochinterpolant:cc}
    I_t = \alpha_t  x_0 + \beta_t  x_1 + \gamma_t  z\qquad t\in[0,1],
\end{equation}
where
\begin{itemize}[leftmargin=0.15in]
\item  $\alpha_t $, $\beta_t $, and $\gamma^2_t$ are differentiable functions of time such that $\alpha_0=\beta_1 =1$,  $\alpha_1=\beta_0 =\gamma_0 = \gamma_1=0$, and $\alpha^2_t+\beta^2_t+\gamma^2_t>0$ for all $t\in[0,1]$.
\item The pair $(x_0,x_1)$ are jointly drawn from a conditional probability density $ \rho(x_0,x_1|\xi)$ such that 
   \begin{align}
       \label{eq:rhoc:cond:cc}
       \int_{\R^d} \rho(x_0,x_1|\xi) dx_1 = \rho_0(x_0|\xi), \\ \int_{\R^d} \rho(x_0,x_1|\xi) dx_0 =  \rho_1(x_1|\xi).
   \end{align}
   \item $z \sim \mathsf N(0,\Id)$, independent of $(x_0,x_1,\xi)$.
\end{itemize}
\end{definition}
Similarly, the functions~\eqref{eq:g:c} become
\begin{align}
    \label{eq:g:cc}
    b_t(x,\xi) &= \E(\dot I_t|I_t = x,\xi), \quad g_t(x,\xi) = \E(z|I_t = x,\xi)
\end{align}
where $\E(\cdot|I_t=x)$ denotes the expectation over $\rho(x_0,x_1|\xi)$ conditional on $I_t=x$, and Theorem~\ref{th:pdf0} becomes:
\begin{restatable}[Transport equation with coupling and conditioning]{theorem}{cont}
    \label{th:pdf}
    The probability distribution of the stochastic interpolant~$I_t$ specified by \cref{def:cond:interp:c} has a density $\rho_t(x|\xi)$ that satisfies $\rho_{t=0}(x|\xi)= \rho_0(x|\xi)$ and $ \rho_{t=1}(x|\xi) = \rho_1(x|\xi)$, and solves the transport equation 
    \begin{align}
        \label{eq:tranp:fpe:cc}
        \partial_t \rho_t(x|\xi)+ \nabla\cdot\left( b_t(x,\xi) \rho_t(x|\xi)\right) = 0,
    \end{align}
    where the velocity field is given in~\eqref{eq:g:cc}.
    Moreover, for every $t$ such that $\gamma_t  \neq 0$, the following identity for the score holds
    \begin{equation}
        \label{eq:s:cc}
        \nabla \log \rho_t(x|\xi)= -\gamma^{-1}_t g_t(x,\xi).
    \end{equation}
    The functions $b$ and $g$ are the unique minimizers of the objective
    \begin{equation}
        \label{eq:obj:cc}
        \begin{aligned}
            L_b(\hat b) &= \int_0^1 \E \left[|\hat b_t(I_t,\xi)|^2 - 2\dot I_t \cdot \hat b_t(I_t,\xi)\right] dt,\\
            L_g(\hat g) &= \int_0^1 \E \left[|\hat g_t(I_t,\xi)|^2 - 2z\cdot \hat g_t(I_t,\xi)\right] dt,
        \end{aligned}
    \end{equation}
    where
    $\E$ denotes an expectation over $(x_0,x_1)\sim \rho(x_0,x_1|\xi)$, $\xi\sim \eta(\xi)$, and $z\sim\mathsf N(0,\Id)$.
\end{restatable}
Note that the objectives~\eqref{eq:obj:cc} can readily be estimated in practice from samples $(x_0,x_1)\sim \rho(x_0,x_1|\xi)$, $z \sim \mathsf N(0, 1)$, and $\xi \sim \eta(\xi)$, which will enable us to learn approximations for use in a generative model.

\begin{proof}
By definition of the stochastic interpolant given in~\eqref{eq:stochinterpolant:cc}, its characteristic function is given by
\begin{equation}
    \label{eq:charact}
    \E [e^{ik\cdot I_t}] = \int_{\R^d\times \R^d} e^{ik\cdot (\alpha_t  x_0+\beta_t  x_1) } \rho(x_0,x_1|\xi) dx_0 dx_1 e^{-\frac12 \gamma^2_t |k|^2},
\end{equation}
where we used $z\perp(x_0,x_1)$ and $z\sim {\sf N}(0,\Id)$.
The smoothness in $k$ of~\eqref{eq:charact} guarantees that the distribution of $I_t$ has a density $\rho_t(x|\xi) > 0$ globally.
By definition of $I_t$, this density $\rho_t(x|\xi)$ satisfies, for any suitable test function $\phi:\R^d \to \R$,
\begin{equation}
    \label{eq:test_def:pdf}
    \begin{aligned}
        &\int_{\R^d} \phi(x) \rho_t(x|\xi) dx= \int_{\R^d\times\R^d\times\R^d} \phi\left(I_t\right) \rho(x_0,x_1|\xi)  (2\pi)^{-d/2} e^{-\frac12 |z|^2} dx_0dx_1dz.
    \end{aligned}
\end{equation}
Above, $I_t= \alpha_t x_0+\beta_t x_1+\gamma_t  z$.
Taking the time derivative of both sides 
\begin{equation}
    \label{eq:test_def:pdf:2}
    \begin{aligned}
        &\int_{\R^d} \phi(x) \partial_t\rho_t(x|\xi) dx\\
        &= \int_{\R^d\times\R^d\times\R^d} \big(\dot \alpha_t x_0+\dot\beta_t x_1+\dot\gamma_t  z\big) \cdot \nabla \phi\left(I_t\right) \rho(x_0,x_1|\xi) (2\pi)^{-d/2} e^{-\frac12 |z|^2} dx_0dx_1dz\\
        &= \int_{\R^d} \E\big[\big(\dot \alpha_t x_0+\dot\beta_t x_1+\dot\gamma_t  z\big) \cdot \nabla \phi(I_t)\big]\big|I_t=x\big] \rho_t(x|\xi) dx\\
        &= \int_{\R^d} \E\big[\dot \alpha_t x_0+\dot\beta_t x_1+\dot\gamma_t  z\big|I_t=x\big] \cdot \nabla \phi(x) \rho_t(x|\xi) dx\
    \end{aligned}
\end{equation}
where we used  the chain rule to get the first equality, the definition of the conditional expectation to get the second, and the tower property $\phi(I_t)=\phi(x)$ conditioned on $I_t=x$ to get the third.  Since
\begin{align}
    \label{eq:b:def:app}
    \E\big[\dot \alpha_t x_0+\dot\beta_t x_1+\dot\gamma_t  z\big|I_t=x\big] =  b_t(x)
\end{align}
by the definition of $b$ in~\eqref{eq:g:cc}, we can therefore write~\eqref{eq:test_def:pdf:2} as
\begin{equation}
    \label{eq:test_def:pdf:3}
    \begin{aligned}
        &\int_{\R^d} \phi(x) \partial_t\rho_t(x|\xi) dx = \int_{\R^d} b_t(x,\xi) \cdot \nabla \phi(x) \rho_t(x|\xi) dx.
    \end{aligned}
\end{equation}
This equation is~\eqref{eq:tranp:fpe:cc} written in weak form.

To establish~\eqref{eq:s:cc}, note that if $\gamma_t >0$, we have
 \begin{equation}
 \label{eq:gbp}
 \begin{aligned}
 \E\big[z  e^{i\gamma_t  k \cdot z }\big] &= -\gamma^{-1}_t(i\partial_k)  \E\big[ e^{i\gamma_t  k \cdot z }\big],\\
 &= -\gamma^{-1}_t(i\partial_k)  e^{-\tfrac12 \gamma^2_t |k|^2},\\
 &= i \gamma_t k e^{-\tfrac12 \gamma^2_t |k|^2}.
 \end{aligned}
 \end{equation}
 As a result, using  $z\perp(x_0,x_1)$, we have
\begin{equation}
 \label{eq:gbp:2}
 \E\big[z  e^{i k \cdot I_t }\big] = i \gamma_t k \E \big[e^{ik\cdot I_t}\big].
 \end{equation}
 Using the properties of the conditional expectation, the left-hand side of this equation can be written
\begin{equation}
\label{eq:gbp:3}
\begin{aligned}
\E\big[z  e^{i k \cdot I_t }\big] & = \int_{\R^d} \E\big[z  e^{i k \cdot I_t }\big|I_t=x\big] \rho_t(x|\xi) dx, \\
& = \int_{\R^d} \E[z  |I_t=x] e^{i k \cdot x } \rho_t(x,\xi) dx, \\
& = \int_{\R^d} g_t(x,\xi) e^{i k \cdot x } \rho_t(x,\xi) dx,
\end{aligned}
\end{equation}
where we used the definition of $g$ in~\eqref{eq:g:cc} to get the last equality.
 Since the right-hand side of~\eqref{eq:gbp:2} is the Fourier transform of $-\gamma_t  \nabla \rho_t(x|\xi)$, we deduce that
 \begin{equation}
 \label{eq:gbp:4}
 g_t(x,\xi) \rho_t(x|\xi) = -\gamma_t  \nabla \rho_t(x|\xi) = -\gamma_t  \nabla \log \rho_t(x|\xi) \, \rho_t(x|\xi).
 \end{equation}
 Since $\rho_t(x|\xi)>0$, this implies~\eqref{eq:s:cc} when $\gamma_t >0$. 

 Finally, to derive~\eqref{eq:obj:cc}, notice that we can write
 \begin{equation}
        \label{eq:obj:cond}
        \begin{aligned}
            L_b(\hat b) &= \int_0^1 \E \left[|\hat b_t(I_t,\xi)|^2 - 2\dot I_t\cdot \hat b_t(I_t,\xi)\right] dt,\\
            &= \int_0^1 \int_{\R^d}\E \left[|\hat b_t(I_t,\xi)|^2 - 2\dot I_t\cdot \hat b_t(I_t,\xi)|I_t = x\right] \rho_t(x|\xi)dxdt\\
            &= \int_0^1 \int_{\R^d} \left[ |\hat b_t(x,\xi)|^2 - 2\E[\dot I_t|I_t=x]\cdot \hat b_t(x,\xi)\right] \rho_t(x|\xi)dxdt\\
            &= \int_0^1 \int_{\R^d} \left[ |\hat b_t(x,\xi)|^2 - 2b_t(x,\xi)\cdot \hat b_t(x,\xi)\right] \rho_t(x|\xi)dxdt
        \end{aligned}
    \end{equation}
where we used the definition of $b$ in~\eqref{eq:g:cc}. The unique minimizer of this objective function is $\hat b_t(x,\xi)=b_t(x,\xi)$, and we can proceed similarly to show that the unique minimizers of $L_g(\hat g)$ is $\hat g_t(x,\xi)=g_t(x,\xi)$, respectively.
\end{proof}

Theorem~\ref{th:pdf} implies the following generalization of Corollary~\ref{th:gen:mod0}:  
\begin{restatable}[Probability flow and diffusions with coupling and conditioning]{corollary}{pflow}
\label{th:gen:mod} 
The solutions to the probability flow equation 
\begin{equation}
    \label{eq:prob:flow:cc}
    \dot X_t = b_t(X_t,\xi)
\end{equation}
enjoy the property that
\begin{align}
X_{t=1} &\sim \rho_1(x_1|\xi)   
\quad \text{ if } \quad X_{t=0} \sim \rho_0(x_0|\xi)\\
X_{t=0} &\sim \rho_0(x_0|\xi)
\quad \text{ if } \quad X_{t=1} \sim \rho_1(x_1|\xi)
\end{align}
In addition, for any $\eps_t\ge0$, solutions to the forward \gls{sde}
 \begin{equation}
    \label{eq:sde:f:cc}
    d X^F_t = b_t(X^F_t,\xi) dt - \eps_t \gamma^{-1}_tg_t(X^F_t,\xi) dt + \sqrt{2\eps_t} dW_t,
\end{equation}
enjoy the property that
\begin{align}
 X^F_{t=1} \sim \rho_1(x_1|\xi) \quad \text{if} \quad X^F_{t=0} \sim \rho_0(x_0|\xi),
\end{align}
and solutions to the backward \gls{sde}
 \begin{equation}
        \label{eq:sde:r:cc}
        d X^R_t = b_t(X^R_t,\xi) dt + \eps_t \gamma^{-1}_t g_t(X^R_t,\xi) dt + \sqrt{2\eps_t} dW_t,
    \end{equation}
enjoy the property that
\begin{align}
 X^R_{t=0} \sim \rho_0(x_0|\xi) \quad 
 \text{if} \quad
 X^R_{t=1} \sim \rho_1(x_1|\xi).
\end{align}
\end{restatable}
Note that if we additionally draw $\xi$ marginally from $\eta(\xi)$ when we generate the solution to these equations, we can also generate samples from the unconditional $\rho_0(x_0) = \int_D \rho_0(x_0|\xi) \eta(\xi) d\xi$ and $\rho_1(x_1) = \int_D \rho_1(x_1|\xi) \eta(\xi) d\xi$.

\begin{proof} The probability flow \gls{ode} is the characteristic equation of the transport equation~\eqref{eq:tranp:fpe:cc}, which proves the statement about its solutions $X_t$. To establish the statement about the solution of the forward \gls{sde}~\eqref{eq:sde:f:cc}, use  expression~\eqref{eq:s:cc} for $\nabla \log \rho_t(x,\xi)$ together with the identity $ \Delta \rho_t(x,\xi) = \nabla \cdot( \nabla \log \rho_t(x,\xi) \,  \rho_t(x,\xi))$ to write~\eqref{eq:tranp:fpe:cc} as the forward Fokker-Planck equation
\begin{equation}
\label{eq:fpe:c:app}
\partial_t \rho_t(x|\xi)+ \nabla\cdot\left( (b_t(x,\xi) - \eps_t \gamma^{-1}_tg_t(x,\xi)) \rho_t(x|\xi)\right) = \eps_t \Delta \rho_t(x|\xi)
\end{equation}
to be solved forward in time since $\eps_t>0$. To establish the statement about the solution of the reversed \gls{sde}~\eqref{eq:sde:r:cc}, proceed similarly to write~\eqref{eq:tranp:fpe:cc} as the backward Fokker-Planck equation
\begin{equation}
\label{eq:fpe:c:b:app}
\partial_t \rho_t(x|\xi)+ \nabla\cdot\left( (b_t(x,\xi) + \eps_t\gamma^{-1}_t g_t(x,\xi)) \rho_t(x|\xi)\right) = -\eps_t \Delta \rho_t(x|\xi)
\end{equation}
to be solved backward in time since $\eps_t>0$. 
\end{proof}

The generative model arising from Corollary \ref{th:gen:mod0} has an associated transport cost which is the subject of Corollary \ref{th:transport}:
\cost*
\begin{proof}
We have 
\begin{equation}
\label{eq:b:s1}
    \begin{aligned}
    \E_{x_0\sim\rho_0} \big[ |X_{t=1}(x_0) - x_0|^2 \big] & =  \E_{x_0\sim\rho_0}\Big[ \Big|\int_0^1 b_t(X_t(x_0)) dt\Big|^2\Big]\\
    & \le \int_0^1 \E_{x_0\sim\rho_0}\big[ | b_t(X_t(x_0)) |^2\big] dt\\
    & = \E\big[ |b_t(I_t)|^2\big]
\end{aligned}
\end{equation}
where we used the probability flow equation~\eqref{eq:prob:flow} for $X_t$ and the property that the law of $X_t(x_0)$ with $x_0\sim \rho_0$ and $I_t$ coincide.
Using the definition of $b_t(x)$ in~\eqref{eq:g:c} and Jensen's inequality we have that
\begin{equation}
\label{eq:b:s2}
   \E\big[ |b_t(I_t)|^2\big] = \E \big[ \big|\mathbb E[\dot I_t | I_t ] \big| ^2\big] \le \E \big [ \E \big [ |\dot I_t|^2 \big| I_t ] \big] = \E[| \dot I_t |^2 ]
\end{equation}
where the last line is true by the tower property of the conditional expectation. Combining~\eqref{eq:b:s1} and~\eqref{eq:b:s2} establishes the bound in~\eqref{eq:W2:bound}.
\end{proof}

\section{Further experimental details \label{appsec:experimental}}

\paragraph{Architecture} For the velocity model we use the U-net from  \cite{ho2020denoising} as implemented in \href{https://github.com/lucidrains/denoising-diffusion-pytorch/blob/main/denoising_diffusion_pytorch/classifier_free_guidance.py}{lucidrain's denoising-diffusion-pytorch} repository; this variant of the architecture includes embeddings to condition on class labels. We use the following hyperparameters:
\begin{itemize}
    \item Dim Mults: (1,1,2,3,4)
    \item Dim (channels): 256
    \item Resnet block groups: 8
    \item Leanred Sinusoidal Cond: True
    \item Learned Sinusoidal Dim: 32
    \item Attention Dim Head: 64
    \item Attention Heads: 4
    \item Random Fourier Features: False
\end{itemize}

\paragraph{Image-shaped conditioning in the Unet.} For image-shaped conditioning, we follow \cite{ho2022cascaded} and append upsampled low-resolution images to the input $x_t$ at each time step to the velocity model. We also condition on the missingness masks for in-painting by appending them to $x_t$. 

\paragraph{Optimization.} We use Adam optimizer 
\citep{kingma2014adam}, starting at learning rate 
2e-4 with the StepLR scheduler which scales the learning rate by $\gamma=.99$ every $N=1000$ steps. We use no weight decay. We clip gradient norms at $10,000$ (this is the norm of the entire set of parameters taken as a vector, the default type of norm clipping in PyTorch library). 

\paragraph{Integration for sampling} We use the Dopri solver from the torchdiffeq library  \citep{torchdiffeq}.

\paragraph{Miscellaneous}
We use Pytorch library along with Lightning Fabric to handle parallelism.

\paragraph{Below we include additional experimental illustrations in the flavor of the figures in the main text. }

\begin{figure*}
    \centering
    \includegraphics[width=0.90\linewidth]{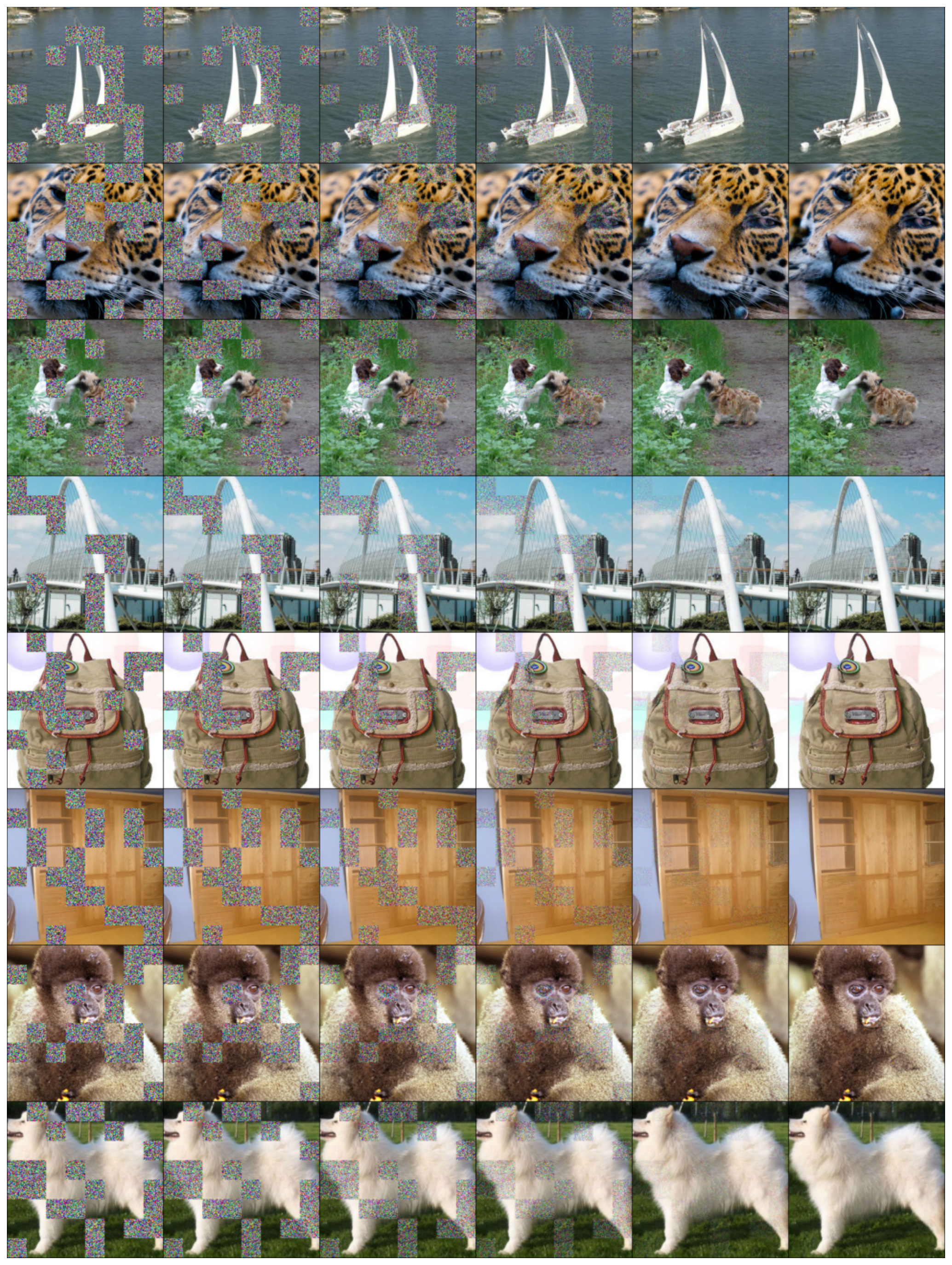}
    \caption{Additional examples of in-filling on the $256\times256$ resolution images, with temporal slices of the probability flow.}
    \label{fig:mask:app:many}
\end{figure*}

\begin{figure*}
    \centering
    \includegraphics[width=0.65\linewidth]{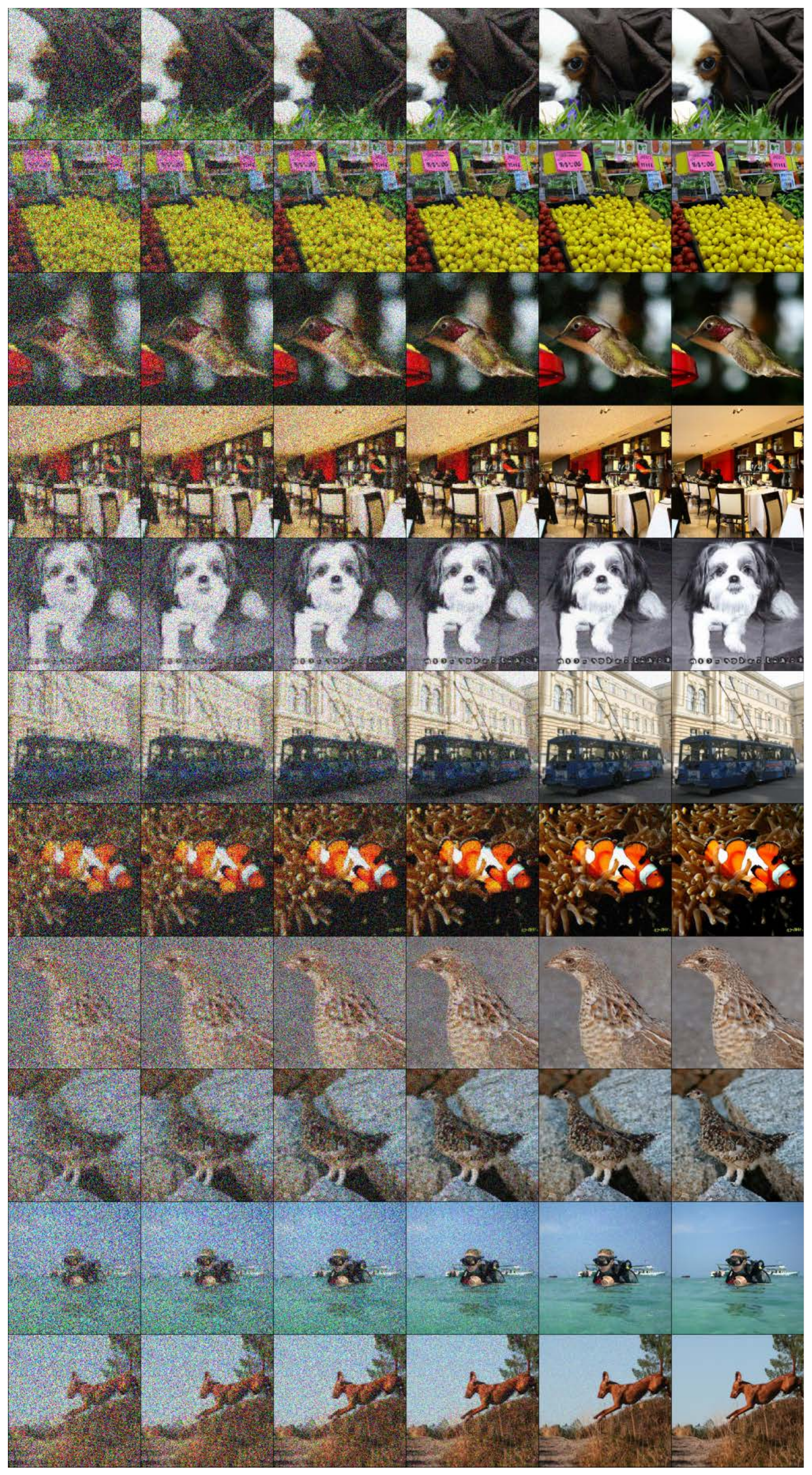}
    \caption{Additional examples of super-resolution
    from $64$ to $256$, with temporal slices of the probability flow.}
    \label{fig:super:app:many}
\end{figure*}

\begin{figure}[t]
    \centering
    \includegraphics[width=0.49\linewidth]
    {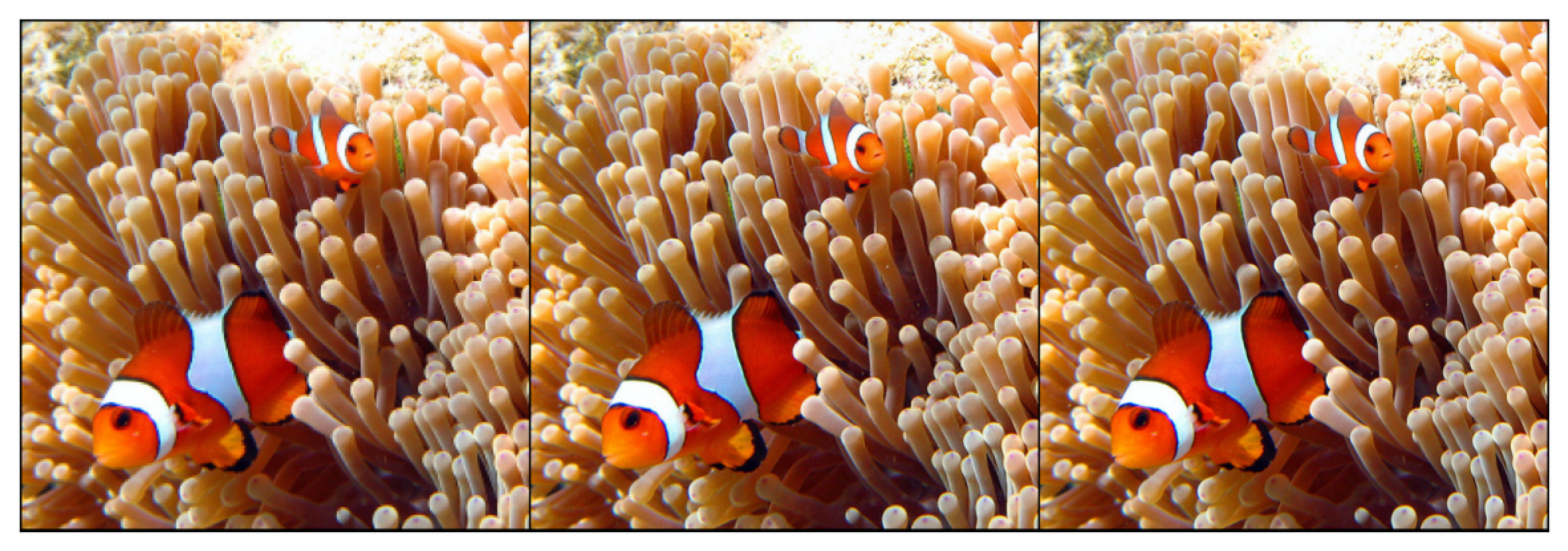}
     \includegraphics[width=0.49\linewidth]
    {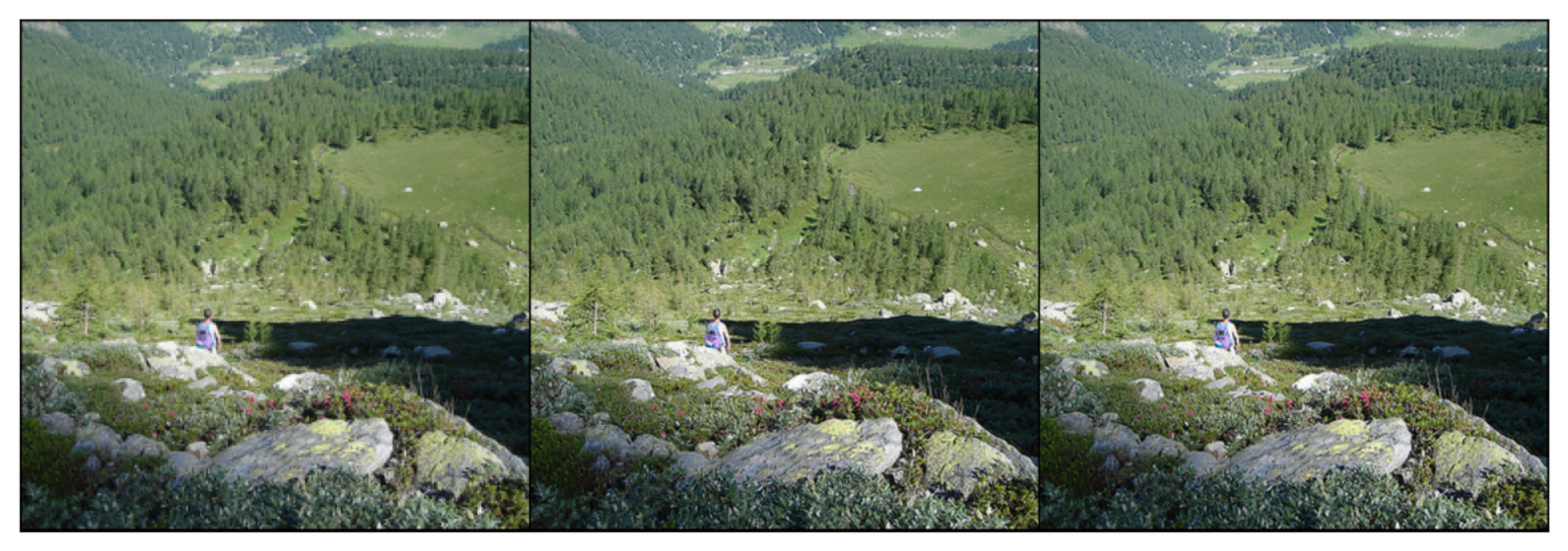}
     \includegraphics[width=0.49\linewidth]
    {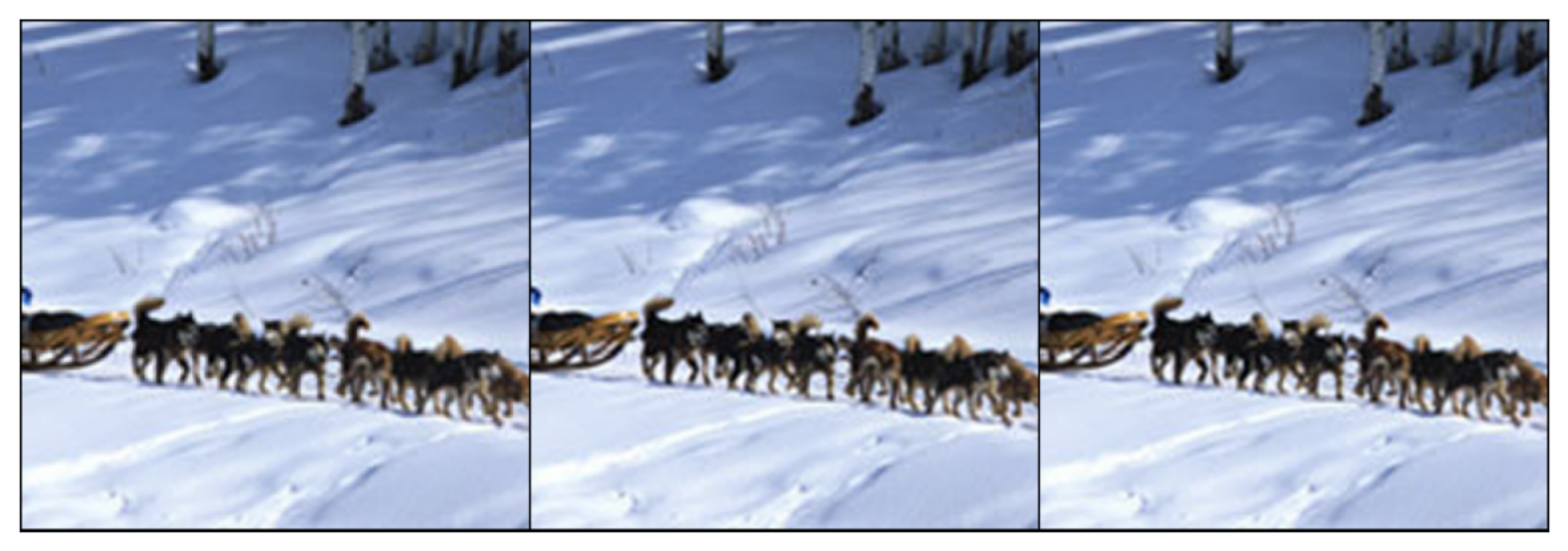}
     \includegraphics[width=0.49\linewidth]
    {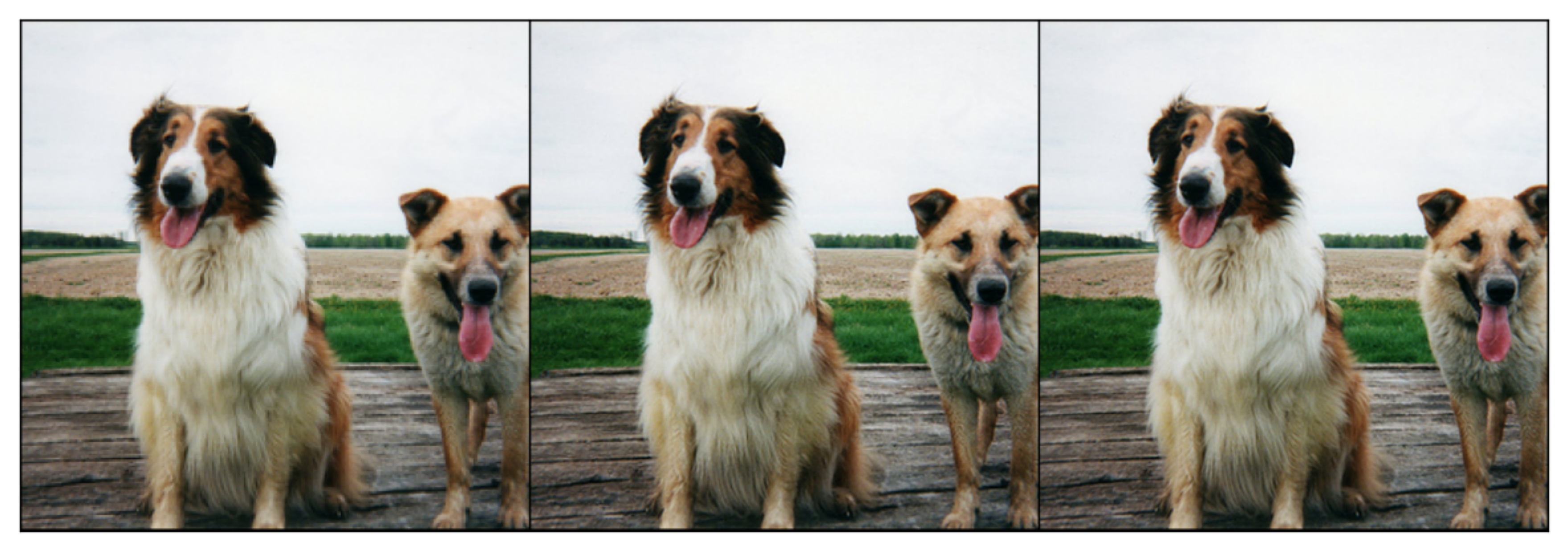}
    \caption{\textbf{Super-resolution:}
    \textit{Top four rows}: Super-resolved images from resolution $256\times 256 \mapsto 512\times 512$, where the left-most image is the lower resolution version, the middle is the model output, and the right is the ground truth.}
    \label{fig:super512}
\end{figure}


\end{document}